%% file: main.tex
\documentclass{article} 
\usepackage{iclr2027_conference,times}
\input{math_commands.tex}

\usepackage{hyperref}
\usepackage{url}
\usepackage{booktabs} 
\usepackage{graphicx}
\usepackage{multirow}
\usepackage{amsmath,amssymb}
\usepackage{algorithm}
\usepackage{algpseudocode}
\usepackage{amsthm}
\usepackage{wrapfig}
\usepackage{subcaption}
\usepackage{xcolor}
\definecolor{slowred}{RGB}{190,45,45}
\definecolor{fastgreen}{RGB}{0,125,65}

\newcommand{\slow}[1]{\textcolor{slowred}{#1}}
\newcommand{\fast}[1]{\textcolor{fastgreen}{#1}}

\newtheorem{theorem}{Theorem}
\newtheorem{proposition}[theorem]{Proposition}

\newtheorem{lemma}[theorem]{Lemma}
\newtheorem{corollary}[theorem]{Corollary}
\title{Moments-guided Edge Sampling}

\iclrfinalcopy

\author{Weibin Cai, Reza Zafarani\\
Data Lab, Department of EECS\\
Syracuse University\\
Syracuse, NY, USA \\
\texttt{\{weibin44,reza\}@data.syr.edu} \\
}

\begin{document}

\maketitle
\lhead{Preprint.}
\vspace{-7mm}
\begin{abstract}
Edge sampling makes local decisions to achieve graph-level objectives, such as preserving structural properties. 
This creates a fundamental challenge: \textit{how can the effect of a local edge edit (i.e., edge addition or removal) on global graph structure be quantified and controlled?} 
We address this challenge with a \textit{moment-guided edge sampling framework} based on spectral moments of the random-walk transition matrix. 
We compute exact moment changes through two complementary methods: a combinatorial method with closed-form updates for low-order moments, and a low-rank method that exploits \textit{locality} and \textit{cyclic trace invariance} to compress computations to edited endpoints, supporting arbitrary moment orders and batched edits. 
For single-edge edits at fixed moment orders, the low-rank method reduces the cost from $O(mn)$ to $O(m)$, while the combinatorial method evaluates low-order changes in constant time given maintained local statistics. 
These moment changes provide \textbf{interpretable structural signatures} of local edge motifs that aggregate into graph-level fingerprints. 
This structural meaning motivates us to ask whether preserving moments also preserves the graph properties. 
We further derive and validate that moment-preserving sampling can \textbf{retain related structural properties}, including triangle-weighted clustering coefficient. 
These structural insights enable \textbf{analysis and improvement of graph learning}: different edge structures have distinct effects on supervised node classification, while moment-guided augmentation is competitive for graph contrastive learning. 
Together, these findings establish moments as an interpretable and controllable bridge from local edge edits to global graph structure and learning. 

\end{abstract}

\input{Sections/1_Introduction}

\input{Sections/Method}

\input{Sections/Experiments}

\input{Sections/Related_Work}

\section{Conclusion}
We introduce a moment-guided edge sampling framework that connects local edge edits to global graph structure through spectral moments. 
To compute edge-wise moment changes efficiently, we develop an $O(1)$ combinatorial method for low-order moment and a low-rank method that supports arbitrary orders and batched edits, reducing single-edge computation by a factor of $n$ over direct recomputation for fixed moment orders. 
These moment changes provide interpretable structural signatures, capturing local motifs and forming graph-level fingerprints. 
Using them, our framework selects edits that move the graph toward a target moment profile. 
We further derive connections between moments and structural properties and validate that moment-preserving sampling retains properties such as triangle-weighted clustering. 
Finally, case studies in supervised node classification and graph contrastive learning suggest that different structural edge types have distinct effects on supervised performance, while moment-guided augmentation yields competitive results in contrastive learning.  
Overall, spectral moments provide an interpretable and controllable bridge from local edge edits to global graph structure and downstream learning.

\bibliography{reference}
\bibliographystyle{iclr2027_conference}

\appendix
\input{Sections/Appendix}

\input{Sections/pseudocode}

\end{document}

%% file: math_commands.tex
\usepackage{amsmath,amsfonts,bm}

\def\eqref#1{equation~\ref{#1}}

\def\1{\bm{1}}

\DeclareMathAlphabet{\mathsfit}{\encodingdefault}{\sfdefault}{m}{sl}
\SetMathAlphabet{\mathsfit}{bold}{\encodingdefault}{\sfdefault}{bx}{n}



%% file: Sections/1_Introduction.tex
\section{Introduction}
Graph sampling is fundamental to graph mining and graph learning. 
In graph sparsification, it reduces graph size while preserving structural properties such as spectra and clustering coefficients~\citep{leskovec2006sampling,ahmed2013network,spielman2008graph}. 
In supervised GNNs, sampling can regularize training, mitigate overfitting and over-smoothing~\citep{rong2019dropedge}, and improve scalability~\citep{hamilton2017inductive,chiang2019cluster,zeng2019graphsaint}. 
In unsupervised GNNs, contrastive learning methods commonly use sampling to construct augmented graph views~\citep{you2020graph,velickovic2018deep,zhu2020deep}. 
Across these settings, edges are often the basic sampling units, while the objectives are defined at the graph level, such as preserving structural properties. 
Targeted edge sampling therefore \textit{requires a way to relate individual edges to global graph structure}. 
However, an edge cannot be characterized independently of its surroundings: its structural role depends on nearby edges, and its contribution to global graph properties is difficult to measure directly. 
What is needed is an \textit{edge-level quantity that is locally computable yet directly reflects how an edit changes global structure}. 
Without such a connection, many commonly used sampling methods rely on simple strategies such as uniform edge dropping or random-walk sampling~\citep{leskovec2006sampling,zeng2019graphsaint}, rather than explicitly controlling the structural information induced by sampling.

Spectral moments of the random-walk transition matrix provide a natural bridge between local edge edits (i.e., edge additions or removals) and global graph structure. 
The $k$-th moment is the average $k$-step return probability over all nodes, or equivalently, a normalized sum of weighted length-$k$ closed walks. 
This connection is useful in two ways. 
(1) \textit{Moments summarize global graph structure}. 
A set of moments maps each graph to a point in a moment space, providing a compact representation of its structure. 
Graphs with similar structural characteristics tend to occupy nearby regions of this space~\citep{jin2020spectral}. 
(2) \textit{Moment changes connect individual edge edits to this global structure representation}. 
Adding or removing an edge changes the weighted closed walks in the graph and therefore moves its point in the moment space. 
This displacement is captured by the edge-wise moment change. 
By computing these changes for candidate edge edits, we can select local edge edits that move the graph toward a desired target point, turning local modifications into controllable changes of global graph structure.

\paragraph{Present work.} 
Building on this view, we develop a moment-guided edge sampling framework that iteratively selects edge edits whose moment changes move the graph toward a target moment profile. 
The main challenge is to compute these edge-wise moment changes efficiently. 
We first develop a \textit{combinatorial method} that computes low-order moment changes by tracking affected closed walks, using maintained statistics to achieve $O(1)$ time per candidate. 
However, this approach becomes impractical at higher orders, as the number and variety of affected closed-walk patterns grow rapidly and make explicit tracking increasingly complex. 
To overcome this limitation, we develop a \textit{low-rank method} that exploits  \textit{locality} and \textit{cyclic trace invariance} to compress affected walks into small endpoint-to-endpoint matrices. 
The low-rank method supports arbitrary moment orders and batched edits, while reducing single-edge computation by a factor of $n$ over direct recomputation. 
We further characterize how moment changes interact across edge edits, providing exact locality and upper bounds on pairwise interactions for arbitrary moment orders in Appendix~\ref{ap:method_details}. 
Using this framework, we show that edge-wise moment changes provide interpretable structural signatures, that preserving moments retains related structural properties, and that this structural control can benefit graph learning. 
Code is available at \href{https://github.com/Weibin44/Moment-guided-graph-sampling}{github}. 
Our contributions are as follows:
\begin{itemize}
    \item We propose an \textbf{efficient moment-guided edge sampling framework} that steers graphs toward target moment profiles. 
    Our combinatorial method computes exact low-order changes in $O(1)$ time per candidate, while our low-rank method support arbitrary orders and batched edits with an $n$-fold reduction over recomputation for single-edge edits. 
    \item We show that edge-wise moment changes provide \textbf{interpretable structural signatures}, characterizing edge motifs and forming graph-level fingerprints. 
    We derive connections between moments and structural properties, including triangle-weighted clustering coefficient, and validate that moment-preserving sampling can \textbf{retain structural properties}. 
    \item We explore the potential of moment-guided sampling through case studies on supervised node classification and graph contrastive learning. 
    Moment changes reveal how different structural edge types affect supervised performance, while moment-guided augmentation achieves competitive results in contrastive learning. 
    These results suggest that moment-guided structural control can support \textbf{analysis and improvement of graph learning}. 
\end{itemize}
\section{Preliminaries}
\paragraph{Random Walk Transition Matrix.} 
For an undirected graph $G = (V, E)$ with $|V| = n$ nodes and $|E|=m$ edges, let $A\in \mathbb{R}^{n \times n}$ denotes its adjacency matrix, where $A_{ij}=1$ if $(i,j)\in E$ and $A_{ij} = 0$ otherwise. 
Let $D\in \mathbb{R}^{n \times n}$ be the diagonal degree matrix, with $D_{ii}=\sum_{j=1}^{n} A_{ji}$. 
The random-walk transition matrix is $P=AD^{-1}$, where $P_{ij}$ is the probability of moving from node $j$ to node $i$ in one step. 
Since $P$ is similar to the symmetric matrix $P_{\mathrm{sym}}=D^{-1/2}AD^{-1/2}$, its eigenvalues are real and lie in $[-1,1]$.


\paragraph{Spectral Moments.} 
Let $\lambda_1,\ldots,\lambda_n$ be the eigenvalues of $P$. 
The $k$-th spectral moment of $G$ is $m_k=\frac{1}{n}\sum_{i=1}^{n}\lambda_i^{k}=\frac{1}{n}\mathrm{Tr}(P^k)$. 
The diagonal entry $P^k_{ii}$ is the probability that a $k$-step random walk starting at node $i$ returns to $i$. 
Thus, $m_k$ is the average $k$-step return probability over all nodes, or equivalently, a normalized sum of weighted length-$k$ closed walks. 
For example, $m_2=\frac{1}{n}\mathrm{Tr}(P^2)=\frac{2}{n}\sum_{\{i,j\}\in E}\frac{1}{d_id_j}$, $m_3=\frac{1}{n}\mathrm{Tr}(P^3)=\frac{6}{n}\sum_{\{i,j,k\}\in\triangle}\frac{1}{d_id_jd_k}$, where $d_i$ is the degree of node $i$ and $\triangle$ denotes the set of unordered triangles; $m_1=0$ for graphs without self-loops. 
Spectral moments are closely connected to graph diffusion and structural properties.
For example, the heat trace~\citep{tsitsulin2018netlsd} and personalized PageRank trace~\citep{page1998pagerank} can be expressed as weighted sums of moments (Appendix~\ref{ap:relationship_diff_moments}) while moments are also related to spectral and structural graph properties~\citep{preciado2013structural,cohen2018approximating}. 
We mainly use $m_2$ and $m_3$ for interpretation and visualization, while our method supports arbitrary moment orders.

%% file: Sections/Method.tex
\section{Moment-Guided Edge Sampling}
\label{sec:method}
Let $\mathcal{K}=\{k_1,\ldots,k_p\}$ denote the selected moment orders, and let $\mathbf{m}=(m_k)_{k\in\mathcal{K}}$ and $\mathbf{m}^*=(m_k^*)_{k\in\mathcal{K}}$ denote the current and target moment profiles. 
Our goal is to select edge edits that move $\mathbf{m}$ toward $\mathbf{m}^*$, where the target may preserve the original moments or specify a desired structural change. 
We represent an edge edit as $\epsilon=(o,u,v)$, where $o\in\{\textsc{add}, \textsc{delete}\}$ specifies whether the edge $(u,v)$ is added or deleted. 
Let $G^\epsilon$ and $P^\epsilon$ denote the graph and random-walk transition matrix after applying $\epsilon$. The induced change in the $k$-th moment is 
\begin{equation}
\Delta m_k(\epsilon)
=
m_k(G^\epsilon) - m_k(G)
=
\frac{1}{n}
\left[
\operatorname{Tr}\left((P^\epsilon)^k\right)
-
\operatorname{Tr}(P^k)
\right].
\label{eq:delta_moment}
\end{equation}
Thus, $\Delta \mathbf{m}(\epsilon)$ describes how a local edge edit moves the graph in moment space. 
By evaluating these changes for candidate edits, we can select operations that move the current moment profile toward the target. 
The main challenge is to compute these edge-wise moment changes efficiently. 
Directly recomputation requires repeatedly evaluating matrix powers for many candidate edits and is prohibitively expensive on large graphs. 
We address this challenge from two complementary perspectives. 
The \textit{combinatorial view} (Section~\ref{sec:combinatorial_method}) directly tracks the local closed-walk changes induced by an edge edit, such as destroyed, created, or reweighted triangles, yielding exact closed-form updates for low-order moments. 
However, explicit enumeration becomes increasing difficult at higher orders as the number of closed-walk patterns grows rapidly. 
The \textit{algebraic low-rank view} (Section~\ref{sec:algebraic_method}) avoids explicit enumeration by exploiting \textit{locality} and \textit{cyclic trace invariance}.  
It represents edge edits as low-rank updates to $P$ and compress affected closed walks into small endpoint-to-endpoint matrices, supporting arbitrary moment orders and batched edits. 
For a single-edge edit, this reduces the cost from $O(Kmn)$ under direct recomputation to $O(Km)$. 
Finally, Section~\ref{sec:selection} uses these moment changes to iteratively select edits that move the graph toward the target profile. 

\subsection{Combinatorial Moment Change}
\label{sec:combinatorial_method}
The key idea is to \textit{express moments as weighted sums of simple closed-walks structures}. 
In particular, $m_2$ is determined by edges, and $m_3$ is determined by triangles. 
We first derive the exact moment changes for deleting a single edge; the corresponding formulas for edge addition are provided in Appendix~\ref{ap:combinatorial_delta}. 
Define 
\begin{equation}
S_2=\sum_{\{i,j\}\in E}\frac{1}{d_i d_j},
\qquad
S_3=\sum_{\{i,j,k\}\in\triangle}\frac{1}{d_i d_j d_k}.
\end{equation}
where $\triangle$ denotes the set of unordered triangles. 
Since $m_2=\frac{2}{n}S_2$ and $m_3=\frac{6}{n}S_3$, computing their changes reduces to updating $S_2$ and $S_3$. 

To avoid rescanning the entire graph after each edge edit, we maintain three local statistics: 
\begin{equation}
    W_i=\sum_{j\in\mathcal{N}(i)}\frac{1}{d_j}, 
    \qquad 
    M_{ij}=\sum_{k\in\mathcal{N}(i)\cap\mathcal{N}(j)}\frac{1}{d_k}, 
    \qquad 
    H_i=\sum_{\substack{
        \{a,b\}\subseteq\mathcal{N}(i)\\ 
        \{a,b\}\in E
    }}\frac{1}{d_a d_b}.
\end{equation}
Here, $W_i$ summarizes the inverse-degree mass of the neighbors of $i$, $M_{ij}$ the weighted common-neighbor mass of $i$ and $j$, and $H_i$ the weighted triangle mass around $i$. 
All quantities are evaluated before the edge edit. 
Once maintained, they allow the moment changes to be computed using only local information around the edited edge.

\paragraph{Edge deletion induced $\Delta m2$.} 
Consider deleting an edge $(u,v)\in E$. 
The deletion removes the contribution of $(u,v)$ from $S_2$ and reweights the surviving edges incident to $u$ and $v$ because their degrees decrease. 
For example, after $d_u$ changes to $d_u-1$, each surviving edge $(u,k)$ with $k\neq v$ gains $\frac{1}{(d_u-1)d_k}-\frac{1}{d_u d_k}=\frac{1}{d_k}\cdot\frac{1}{d_u(d_u-1)}$. 
Summing these changes over both endpoints and subtracting the removed-edge contribution gives 
\begin{equation}
\Delta S_2^{\mathrm{del}}(u,v)
=
\underbrace{
\frac{W_u-d_v^{-1}}{d_u(d_u-1)}
+
\frac{W_v-d_u^{-1}}{d_v(d_v-1)}
}_{\text{reweighted surviving edges}}
-
\underbrace{
\frac{1}{d_u d_v}
}_{\text{destroyed edge}}.
\label{eq:combinatorial_s2}
\end{equation}


\paragraph{Edge deletion induced $\Delta m3$.} 
Deleting $(u,v)\in E$ similarly has two effects on $S_3$: it destroys all triangles containing $(u,v)$ and reweights the surviving triangles incident to $u$ or $v$. 
For endpoint $u$, the total triangle mass before deletion is $H_u$, while the mass of triangles containing $(u,v)$ is $M_{uv}/d_v$.  
Thus, the surviving triangle mass around $u$ is $H_u - M_{uv}/d_v$, and its reweighting contributes $\left(\frac{1}{d_u-1}-\frac{1}{d_u}\right)\left(H_u-\frac{M_{uv}}{d_v}\right)$. 
Combining the reweighted surviving triangles at both endpoints with the destroyed triangles gives 

\begin{equation}
\Delta S_3^{\mathrm{del}}(u,v)
=
\underbrace{
\frac{H_u-M_{uv}/d_v}{d_u(d_u-1)}
+
\frac{H_v-M_{uv}/d_u}{d_v(d_v-1)}
}_{\text{reweighted surviving triangles}}
-
\underbrace{
\frac{M_{uv}}{d_u d_v}
}_{\text{destroyed triangles}}.
\end{equation}

The corresponding moment changes are
\begin{equation}
\Delta m_2
=
\frac{2}{n}\Delta S_2,
\qquad
\Delta m_3
=
\frac{6}{n}\Delta S_3.
\end{equation}
Once $d_i$, $W_i$, $M_{ij}$, and $H_i$ are maintained, evaluating $\Delta m_2$ and $\Delta m_3$ for a candidate edge requires only constant-time lookups and arithmetic. 
However, extending this combinatorial approach to higher-order moments requires tracking an increasing variety of affected closed-walk patterns. This motivates the low-rank formulation in Section~\ref{sec:algebraic_method}.

\subsection{Low-rank Moment Change}
\label{sec:algebraic_method}
We now derive an algebraic formulation for exact moment changes of arbitrary order and batched edge edits. 
To obtain such a formulation, we exploit two properties of edit-affected closed walks: \textit{locality}, since an edge edit changes the transition matrix only at its endpoints, and \textit{cyclic trace invariance}, which allows affected closed walks to be anchored at these touched nodes. 
Together, these properties lead to a low-rank update of the transition matrix and a compact trace formulation over small endpoint-to-endpoint matrices. 

\noindent \textbf{Locality: edge edits form a low-rank update.}  
Let $\mathcal{B}$ be a batch of edge edits, and let $G'$ and $P'$ denote the graph and transition matrix after applying all edits in $\mathcal{B}$. 
Only the transition columns corresponding to edited endpoints can change. 
We collect these nodes into the touched endpoint set 
\begin{equation}
T=\{i\in V:\text{$i$ is an endpoint of an edit in $\mathcal{B}$}\},
\qquad
r=|T|\leq 2|\mathcal{B}|.
\end{equation}
Write $T=\{i_1,\ldots,i_r\}$. 
For each touched node $i_a$, let $p_a=P_{:i_a}$ and $p'_a=P'_{:i_a}$ denote its transition columns before and after the edits, and define $c_a=p'_a-p_a$. 
Stacking these columns gives $C=[c_1,\ldots,c_r]\in\mathbb{R}^{n\times r}$ and $R=[e_{i_1},\ldots,e_{i_r}]\in\mathbb{R}^{n\times r}$, where $e_i$ is the $i$-th standard basis vector. 
The updated transition matrix can therefore be written exactly as 
\begin{equation}
    P' = P + CR^\top.
\end{equation}
Thus, a batch of edge edits induces a rank-$r$ update to transition matrix $P$, where $r$ depends only on the number of touched endpoints. 

\noindent \textbf{Cyclic trace invariance: affected closed walks are anchored at touched endpoints.} 
Let $Q = CR^\top \in \mathbb{R}^{n\times n}$.   
Expanding $(P+Q)^k$ and canceling the all-$P$ term gives 
\begin{equation}
\Delta T_k
=
n\Delta m_k
=
\operatorname{Tr}\left((P+Q)^k\right)-\operatorname{Tr}(P^k)
=
\sum_{w\in\{P,Q\}^k\setminus\{P^k\}}
\operatorname{Tr}(w). 
\end{equation}
Each word $w$ contains at least one $Q$ and represents a class of affected closed walks.  
Intuitively, a $Q$ factor marks a position in the closed walk where the transition is changed at a touched endpoint, while a $P$ factor represents an unchanged transition.  
These affected closed walks can therefore be characterized by two factors: 
(1) how many touched-endpoint occurrences they contain, and
(2) where these occurrences appear along the closed walk. 
These correspond to the number and positions of the $Q$ factors in $w$. 

Let $\ell$ denote the number of $Q$ factors. 
Grouping the affected closed walks by $\ell$ gives 
\begin{equation}
\Delta T_k
=
\sum_{\ell=1}^{k}
\sum_{\substack{
w\in\{P,Q\}^k\\
w\text{ contains }\ell\text{ copies of }Q
}}
\operatorname{Tr}(w).
\label{eq:group_by_q}
\end{equation}
We next describe the positions of these $\ell$ touched-endpoint occurrences. 
By cyclic invariance of the trace, each closed walk can be rotated so that it starts from a $Q$ factor. 
It can then be written as 
\begin{equation}
\operatorname{Tr}
\left(
P^{a_1}Q
P^{a_2}Q
\cdots
P^{a_\ell}Q
\right),
\label{eq:anchored_word}
\end{equation}
where $a_j$ is the number of unchanged transitions between two consecutive touched-endpoint occurrences. 
Since the closed walk has total length $k$, with $\ell$ positions occupied by $Q$, these segment lengths satisfy $a_1+\cdots+a_\ell=k-\ell$. 
Thus, all possible positions of the $\ell$ touched-endpoint occurrences are represented by 
\begin{equation}
\mathcal{A}_{k,\ell}
=
\left\{
(a_1,\ldots,a_\ell)\in\mathbb{Z}_{\geq0}^{\ell}
:
\sum_{j=1}^{\ell}a_j=k-\ell
\right\}.
\label{eq:segment_configurations}
\end{equation}
Cyclic rotation also changes the counting multiplicity. 
In the original trace expansion, a length-$k$ closed walk can be read from any of its $k$ positions. 
After anchoring the walk at a touched-endpoint occurrence, only its $\ell$ $Q$ positions can serve as starting points. 
Therefore, each anchored representation is weighted by $k/\ell$, giving 
\begin{equation}
\Delta T_k
=
\sum_{\ell=1}^{k}
\frac{k}{\ell}
\sum_{(a_1,\ldots,a_\ell)\in\mathcal{A}_{k,\ell}}
\operatorname{Tr}
\left(
P^{a_1}Q
P^{a_2}Q
\cdots
P^{a_\ell}Q
\right).
\label{eq:rotated_trace}
\end{equation}



\noindent \textbf{Endpoint compression: affected walks become small matrices product.} 
The representation above still involves the full $n\times n$ matrices $P$ and $Q$. 
Since $Q=CR^\top$ is supported only on touched endpoints, we can further compress each walk segment between two consecutive $Q$ factors to these endpoints. 
Substituting $Q=CR^\top$ into Eq.~\ref{eq:rotated_trace} and applying cyclic trace invariance gives 
\begin{equation}
\operatorname{Tr}\left(P^{a_1}CR^\top P^{a_2}CR^\top\cdots P^{a_\ell}CR^\top\right) =
\operatorname{Tr}\left[
(R^\top P^{a_1}C)(R^\top P^{a_2}C)\cdots(R^\top P^{a_\ell}C)
\right].
\end{equation}

We therefore define  
\begin{equation}
    H_t = R^\top P^t C \in \mathbb{R}^{r\times r}
\end{equation}
Each $H_t$ summarizes an edit-induced transition at one touched endpoint followed by $t$ unchanged transitions that end at another touched endpoint. 
Thus, a closed walk with segment lengths $(a_1,\ldots,a_\ell)$ is compressed into the product $H_{a_1}H_{a_2}\cdots H_{a_\ell}$. 
Eq.~\ref{eq:rotated_trace} therefore becomes 
\begin{equation}
\Delta T_k
=
\sum_{\ell=1}^{k}
\frac{k}{\ell}
\sum_{(a_1,\ldots,a_\ell)\in\mathcal{A}_{k,\ell}}
\operatorname{Tr}
\left(
H_{a_1}H_{a_2}\cdots H_{a_\ell}
\right),
\label{eq:low_rank_delta}
\end{equation}
and the moment change is $\Delta m_k=\Delta T_k/n$.
Hence, instead of computing affected closed walks through full $n\times n$ matrix power, we only need products of $r\times r$ matrices defined over the touched endpoints. 
Implementation details and complexity analysis are provided in Appendix~\ref{ap:method_details}.

\subsection{Moment-Guided Edge Selection}
\label{sec:selection}
Given moment changes computed in Sections~\ref{sec:combinatorial_method} and~\ref{sec:algebraic_method}, we select edge edits that move the current moment profile toward the target $\mathbf{m}^*$. 

\textbf{Scoring.} 
For a candidate edit $\epsilon$, the expected moment profile after applying this edit is $\mathbf{m}+\Delta\mathbf{m}(\epsilon)$. 
We score the edit by its normalized distance to the target:  
\begin{equation}
\operatorname{score}(\epsilon)
=
\sum_{k\in\mathcal{K}}\left(
\frac{m_k+\Delta m_k(\epsilon)-m_k^\ast}{\sigma_k}
\right)^2, 
\label{eq:score}
\end{equation}
where $\sigma_k=m_k(G)$ is the $k$-th moment of the original graph $G$. 
This fixed normalization places different moment orders on comparable scales throughout the sampling process. 

\textbf{Selection.} 
At each step, we evaluate Eq.~\ref{eq:score} over a candidate set $\mathcal{C}_t$ and select
\begin{equation}
\epsilon_t^*
=
\arg\min_{\epsilon\in\mathcal{C}_t}
\operatorname{score}(\epsilon).
\end{equation}
We apply $\epsilon_t^*$, update $\mathbf{m}\leftarrow\mathbf{m}+\Delta\mathbf{m}(\epsilon_t^*)$, and repeat until the sampling budget is reached. 
The candidate set can include both edge additions and deletions.

%% file: Sections/Experiments.tex
\section{moment changes as structural signatures}
\label{sec:graph_summarization}

Having established how to compute edge-wise moment changes, we next examine what structural information they encode. 
For interpretability, we focus on edge deletion. 
For an existing edge $e=(u,v) \in E$, let $\epsilon_e=(\textsc{Delete},u,v)$ and write $\Delta\mathbf{m}(e):=\Delta\mathbf{m}(\epsilon_e)$ for brevity. 
We show that these moment changes characterize local edge structure and aggregate into graph-level fingerprints. 

\begin{wrapfigure}{r}{0.5\linewidth}
    \centering
    \vspace{-9mm}
    \includegraphics[width=\linewidth]{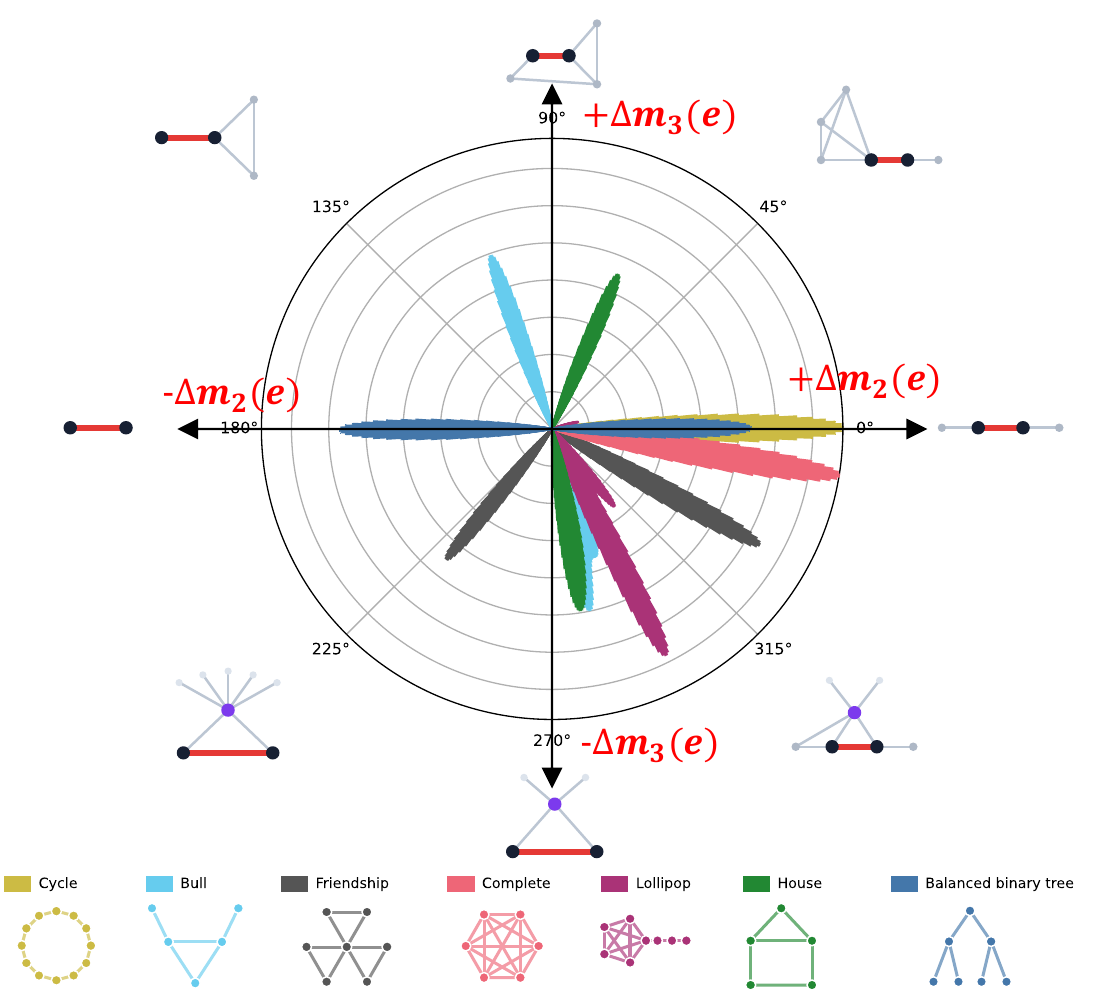}
    \vspace{-7mm}
    \caption{Edge distributions of classic graphs in the $(\Delta m_2, \Delta m_3)$ space. 
    Outer examples show representative edge-centered motifs along different moment-change directions, with the red edge indicating the deleted edge. 
    Inner profiles show the smoothed edge distributions for classic graphs. 
    }
    \vspace{-6mm}
    \label{fig:edge_moment_shifts}
\end{wrapfigure}

\paragraph{Moment changes reveal edge-level structural signatures.}  
The vector $\Delta \mathbf{m}(e)$ describes how deleting $e$ moves the graph in moment space, reflecting the edge's structural role. 
For visualization, we focus on $(m_2,m_3)$ and normalize each coordinate as detailed in Appendix~\ref{ap:edge_moment_direction}. 
Figure~\ref{fig:edge_moment_shifts} connects $\Delta \mathbf{m}(e)$ to graph structure at two levels: representative edge-centered motifs illustrate their local meaning, while classic-graph distributions show how these local patterns aggregate across a graph. 

Edges within triangles tend to have negative $\Delta m_3(e)$ because their deletion destroys triangles. 
In contrast, bridge-like edges attached to clustered regions can have positive $\Delta m_3(e)$ because their removal concentrates three-step return probability within the remaining cluster.  
For $\Delta m_2(e)$, deleting an internal path edge can increase two-step return probabilities, whereas deleting an isolated dyad edge decreases them by isolating both endpoints. 

These patterns are also visible in the classic graphs. 
Graphs with structurally homogeneous edges, such as cycles and complete graphs, concentrate around a narrow range of directions, whereas graphs with multiple edge types, as the bull graph, exhibit multiple directional modes. 
Together, these observations show that moment changes provide interpretable signatures of local edge structure. 

\begin{figure}
    \centering
    \includegraphics[width=\linewidth]{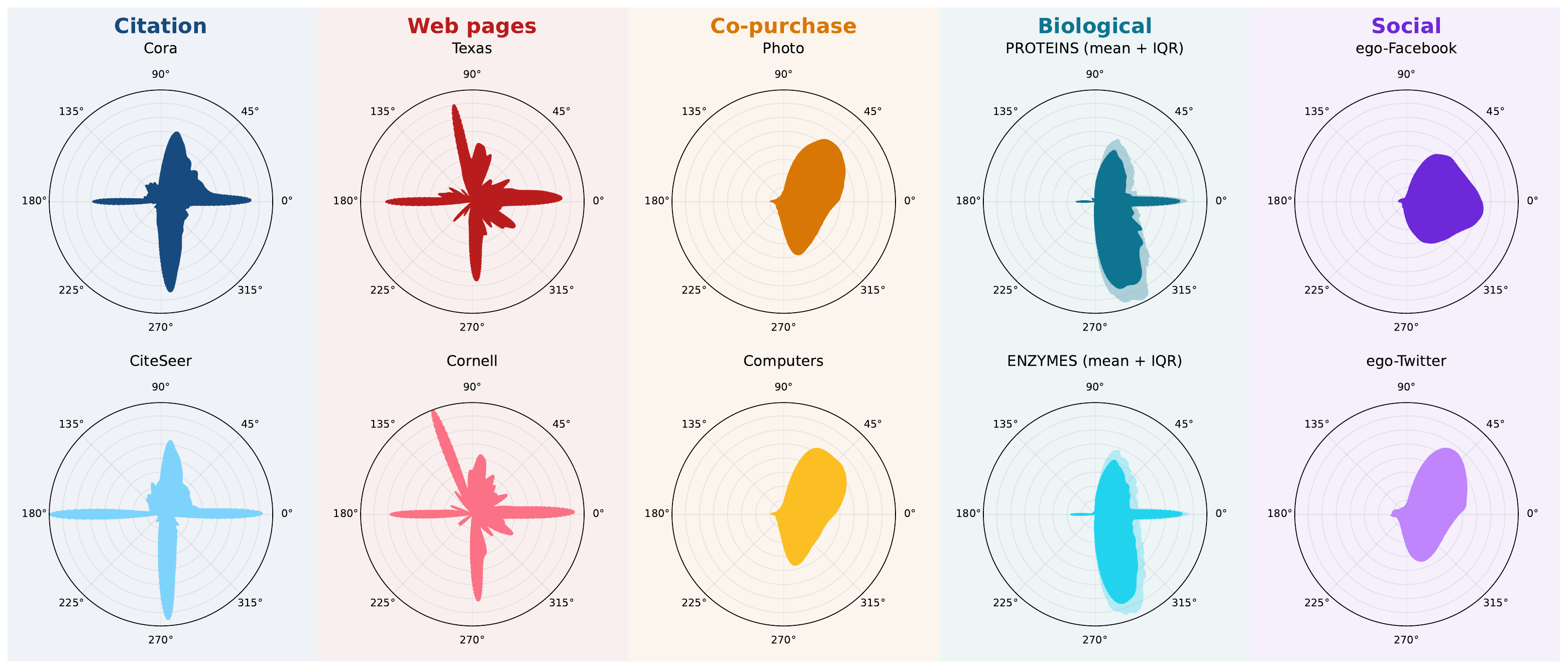}
    \vspace{-6mm}
    \caption{Moment-change fingerprints of real-world graphs. 
    Each profile shows the smoothed edge distribution in the normalized $(\Delta m_2(e), \Delta m_3(e))$ space. 
    Similar profiles indicate similar distributions of local edge-centered motifs. Visualization details are provided in Appendix~\ref{ap:edge_moment_direction}.}
    \label{fig:fingerprint_graphs} 
\end{figure}

\paragraph{Edge-level signatures aggregate into graph-level fingerprints.} 
We next move from individual edges to real-world graphs. 
As shown in Figure~\ref{fig:fingerprint_graphs}, graphs from the same domain often exhibit similar moment-change profiles, while different domains can show distinct patterns. 
These distributions provide compact graph-level fingerprints of the structural edge types present in a graph. 

\noindent \textit{Different fingerprints reveal structural differences across domains.} 
Citation and web-page graphs show broadly similar profiles but differ clearly in the second quadrant. 
Web-page graphs contain more edges in this region, which are often bridge-like edges connected to clustered regions. 
This pattern is consistent with hub pages serving navigational roles across otherwise distinct local page clusters, whereas citation edges more often connect papers within the same topical community.

\noindent \textit{Similar fingerprints suggest shared edge-formation mechanisms.} 
Co-purchase networks and social networks exhibit similar profiles despite belonging to different domains.  
This similarity may reflect shared-neighborhood mechanisms: products connect through common transactions or categories, while people connect through shared affiliations, interests, or mutual friends~\citep{kossinets2006empirical}. 
Such mechanisms promote dense communities, triangles, and hub--periphery patterns~\citep{leskovec2009community}, leading to similar moment-change profiles.

\section{Moment-preserving sampling retain structural Properties}
\label{sec:graph_properties}

Section~\ref{sec:graph_summarization} shows that edge-wise moment changes encode structural information. 
We now ask whether preserving the moments also retains structural properties associated with them. 
Because $P$ normalizes transitions by node degree and $m_k=n^{-1}\operatorname{Tr}(P^k)$ averages over all nodes, \textit{these moments naturally align with normalized or degree-weighted properties rather than raw, size-dependent counts.} 
This makes moment preservation particularly relevant when reducing graph complexity while retaining structural characteristics, as in graph sparsification~\citep{spielman2008graph}, representative graph sampling~\citep{leskovec2006sampling}, and graph augmentation~\citep{rong2019dropedge}. 

To examine this relationship, we perform \textit{moment-preserving edge sampling}. Let $G_0$ be the original graph and set the target profile to its moments, $\mathbf{m}^*=\mathbf{m}(G_0)$. 
We progressively remove edges while minimizing deviation from this target and track which structural properties remain stable. 
We study four properties related to the moments and compare their preservation against uniform random edge removal at the same removal ratios. 
Figure~\ref{fig:preserve_graph_properties} shows representative results, with  additional results and derivations are provided in Appendix~\ref{ap:preserve_graph_properties}.

\begin{figure}[t]
    \centering
    \begin{minipage}{0.24\linewidth}
        \centering
        \includegraphics[width=\linewidth]{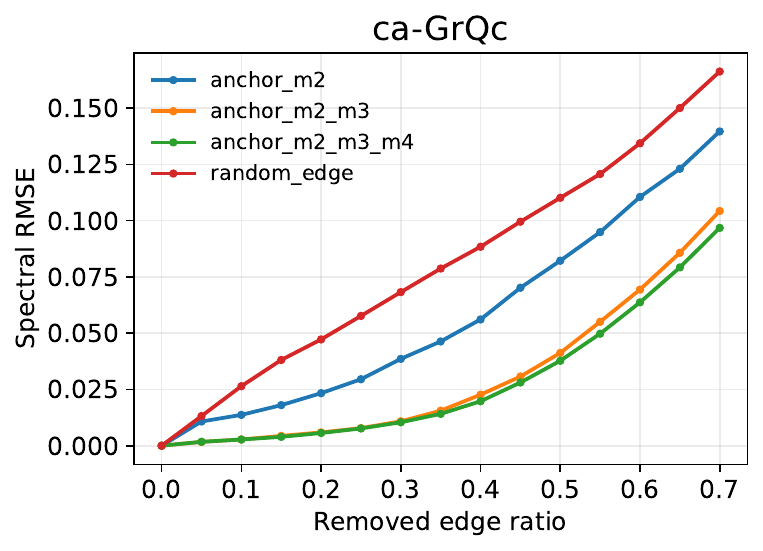}
    \end{minipage}
    \hfill
    \begin{minipage}{0.24\linewidth}
        \centering
        \includegraphics[width=\linewidth]{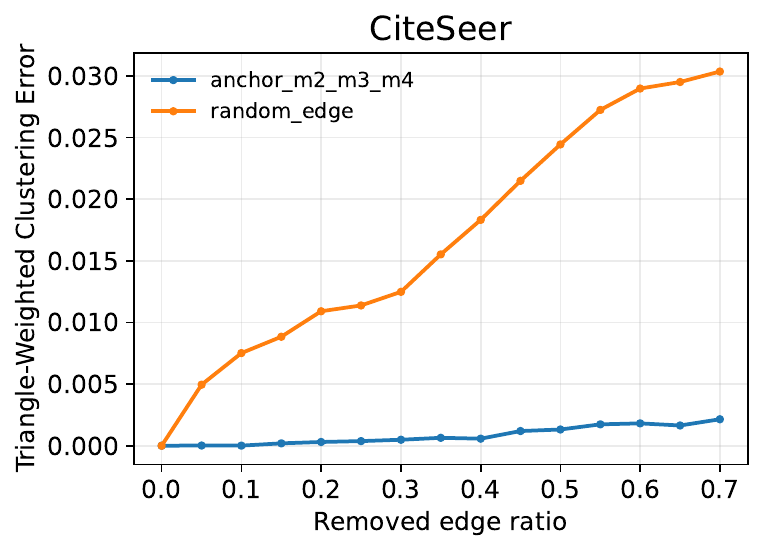}
    \end{minipage}
    \hfill
    \begin{minipage}{0.24\linewidth}
        \centering
        \includegraphics[width=\linewidth]{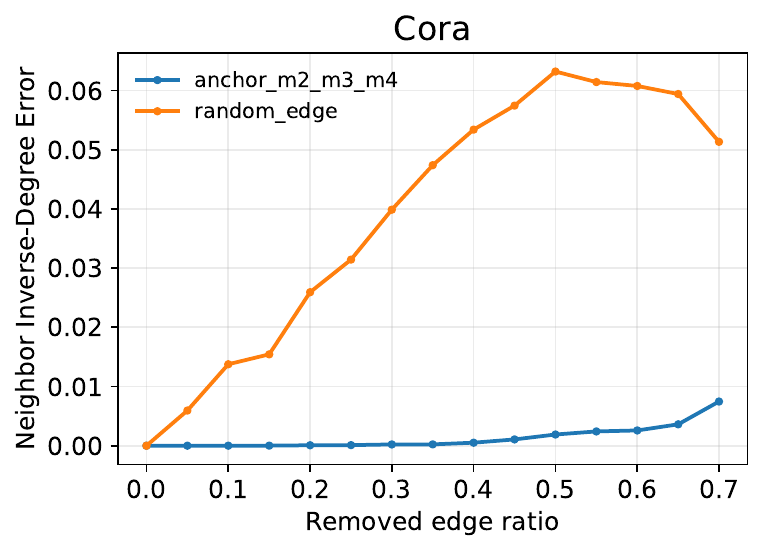}
    \end{minipage}
    \hfill
    \begin{minipage}{0.24\linewidth}
        \centering
        \includegraphics[width=\linewidth]{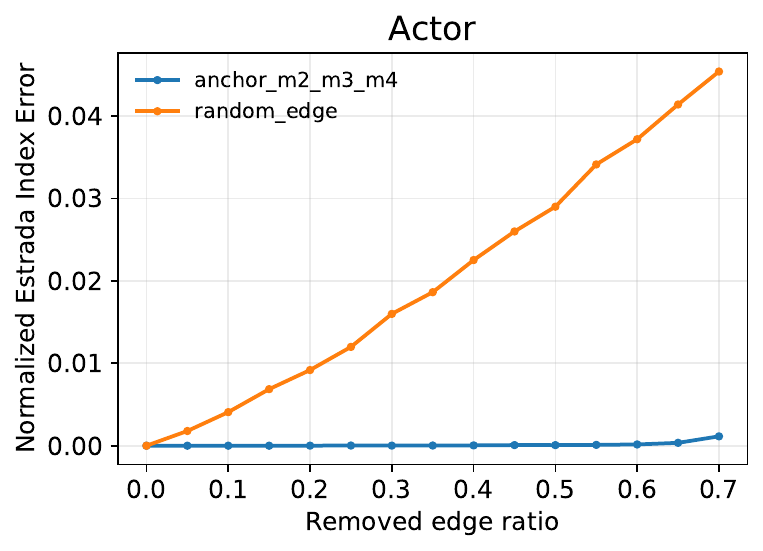}
    \end{minipage}

    \caption{
    Preservation of four structural properties under moment-preserving edge sampling: \textit{normalized spectrum}, \textit{triangle-weighted clustering}, \textit{neighbor inverse degree}, and \textit{normalized Estrada index}. 
    Moment-preserving sampling yields smaller deviations than uniform random edge removal across all four properties. 
    For the spectrum, preserving more moment orders further reduces RMSE. 
    }
    \label{fig:preserve_graph_properties}
\end{figure}

\paragraph{Normalized spectrum.} 
Moments characterize the eigenvalue distribution, and finitely many moments can approximate a graph spectrum~\citep{cohen2018approximating,kong2017spectrum}. 
Preserving more moments constraints the spectrum more strongly. 
We measure spectral preservation using the root mean squared error (RMSE) between the eigenvalues of the original and sampled graphs. 
Figure~\ref{fig:preserve_graph_properties} shows that preserving additional moment orders further reduces the spectral RMSE.

\paragraph{Triangle-weighted clustering coefficient.} 
The third moment $m_3$ is a degree-weighted sum over triangles and is therefore closely related to clustering. 
Let $C_i=\frac{2|\Delta_i|}{d_i(d_i-1)}$, $\Delta_i = \bigl\{\{j,k\}: j,k \in N_i,\ (j,k)\in E \bigr\}$ denote the local clustering coefficient and the set of triangles incident to node $i$. 
This relation motivates the \textit{triangle-weighted} clustering coefficient 
\begin{equation}
\widetilde C_i
=
C_i\cdot
\mathbb{E}_{\{j,k\}\sim \Delta_i}
\left[
\frac{1}{\sqrt{d_j d_k}}
\right],
\label{eq:simplified_clustering}
\end{equation}
Compared with standard clustering, $\widetilde C_i$ assigns greater weight to triangles involving lower-degree neighbors. 
We evaluate preservation using the absolute error of its graph-level mean.

\paragraph{Degree structure: accessibility-weighted inverse degree of $r$-step reachable nodes.} 
The second moment has a direct interpretation in terms of neighboring degrees: $m_2=\frac{1}{n}\sum_i\mathbb{E}_{j\sim N(i)}[1/d_j]$. 
Thus, $m_2$ measures the average inverse degree observed from one-hop neighbors. 

More generally, higher even-order moments extend this interpretation to $r$-step reachable nodes. 
For an $r$-step walk, define 
\begin{equation}
a_i^{(r)}(j)
=
d_i(P^r)_{ji}
=
\sum_{\pi:i\overset{r}{\leadsto}j}
\prod_{v\in\operatorname{Int}(\pi)}\frac{1}{d_v},
\qquad
m_{2r}
=
\frac{1}{n}
\sum_i
\mathbb{E}
\left[
\left.
\frac{a_i^{(r)}(X_r)}{d_{X_r}}
\right|
X_0=i
\right].
\label{eq:degree_structure}
\end{equation}
Here, $X_0=i$ and $X_r$ denote the starting and $r$-step reached nodes, respectively; $\pi:i\overset{r}{\leadsto}j$ denotes an $r$-step walk from $i$ to $j$, and $\operatorname{Int}(\pi)$ its internal nodes. 
Thus, $m_{2r}$ captures the inverse degree of $r$-step reachable nodes weighted by their accessibility. 
In our experiments, we use the $r=1$ case and measure the absolute error of average neighbor inverse degree. 


\paragraph{Diffusion-based connectivity.} 
The moment $m_k$ is the average probability that a $k$-step random walk returns to its starting node. 
This return probability reflects how strongly diffusion remains concentrated around its origin: larger values indicate more locally confined diffusion, while smaller values indicate broader spreading through the graph. 
The collection of moments therefore reflects this diffusion-based connectivity across different walk lengths. 
To aggregate return behavior across different walk lengths, we use a normalized variant of the Estrada index~\citep{jin2020spectral,estrada2002characterization} $\frac{1}{n}EE_P(G)=\frac{1}{n}\sum_{j=1}^{n}e^{\lambda_j}=\frac{1}{n}\operatorname{Tr}(e^P)=\frac{1}{n}\sum_{k=0}^{\infty}\frac{\operatorname{Tr}(P^k)}{k!}=\sum_{k=0}^{\infty}\frac{m_k}{k!}$.  
Thus, the normalized Estrada index summarizes multi-scale return behavior as a diffusion-based connectivity measure. 
We measure its preservation using absolute error.

Figure~\ref{fig:preserve_graph_properties} shows that moment-preserving sampling produces smaller deviations than uniform random edge removal across all four properties. 
Together, these results indicate that controlling spectral moments can retain multiple structural characteristics while the graph is progressively sampled.

\section{moment-guided sampling for graph learning}
\label{sec:improve_gnns}
Section~\ref{sec:graph_summarization} and~\ref{sec:graph_properties} show that moment encode and preserve structural information. 
We next explore the potential of this structural control for graph learning through two case studies on Cora and Citeseer: supervised node classification and graph contrastive learning.

\begin{figure}[t]
    \centering
    \begin{subfigure}{0.42\linewidth}
        \centering
        \includegraphics[width=\linewidth]{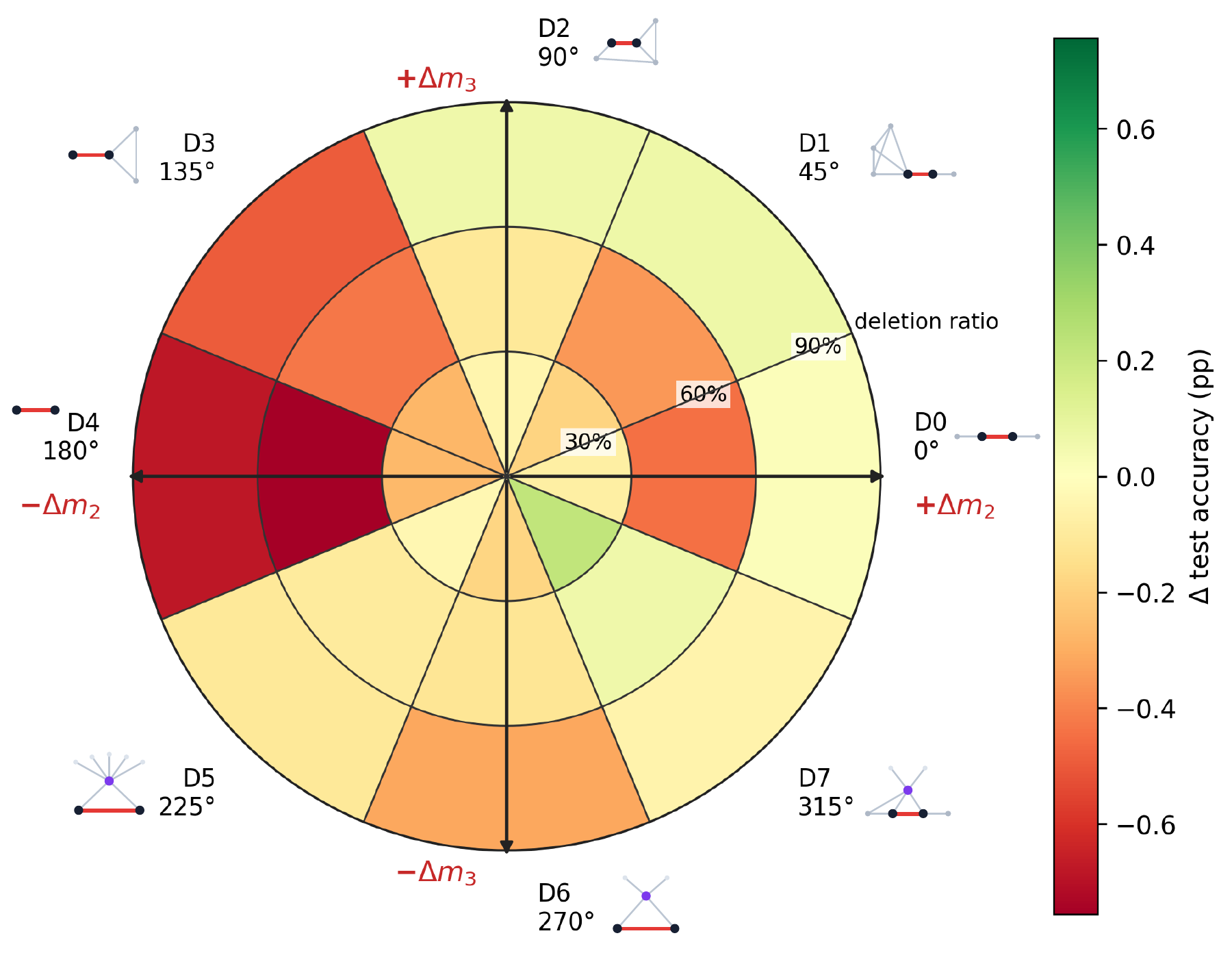}
        \caption{$\Delta$ Test accuracy on Cora.}
        \label{fig:supervised_cora}
    \end{subfigure}
    \hspace{0.12\linewidth}
    \begin{subfigure}{0.42\linewidth}
        \centering
        \includegraphics[width=\linewidth]{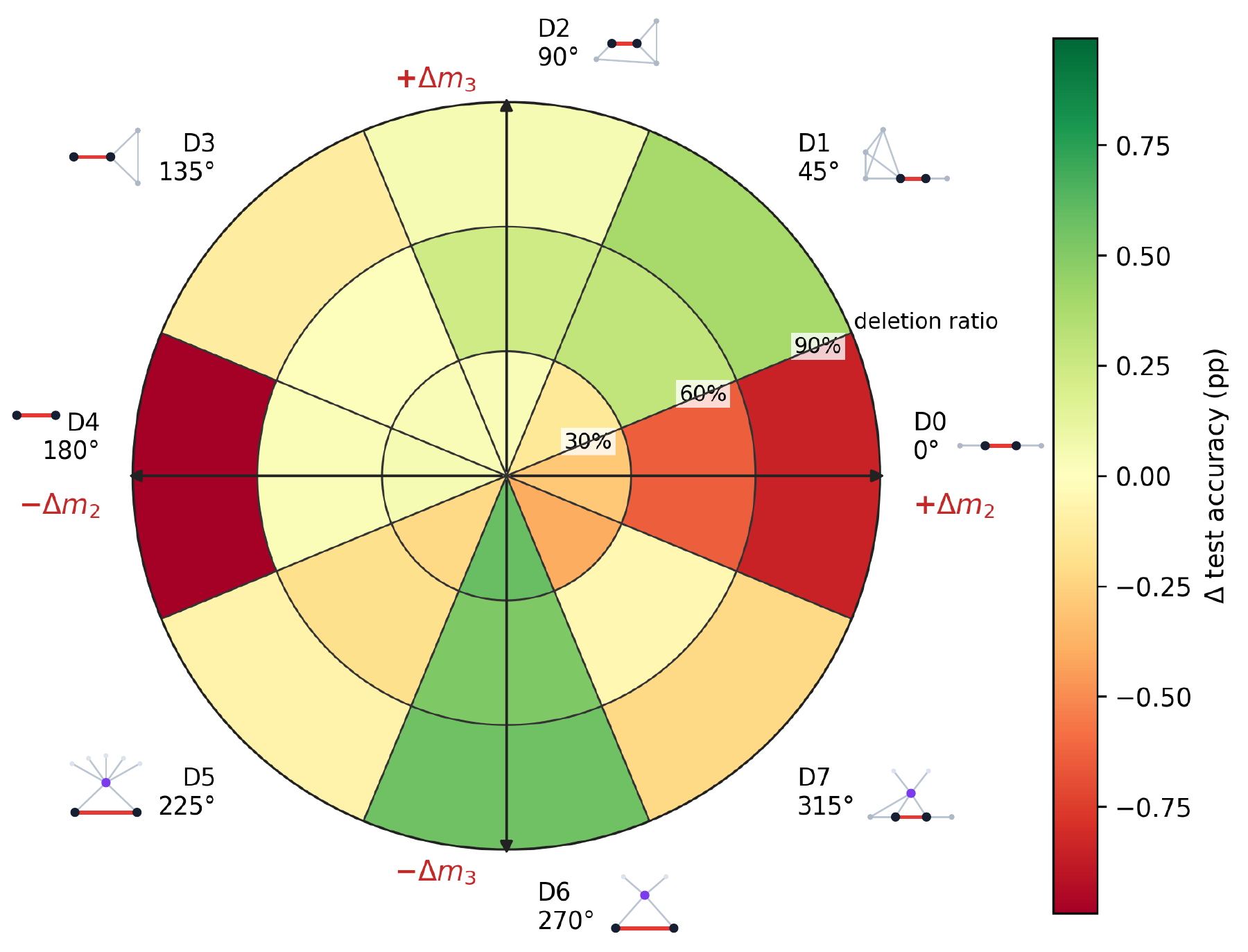}
        \caption{$\Delta$ Test accuracy on Citeseer.}
        \label{fig:supervised_citeseer}
    \end{subfigure}
    \vspace{-2mm}
    \caption{
    Effect of removing edges from different moment-change directions on supervised node classification. 
    Edges are divided into eight directions ($D0$--$D7$), each associated with a representative motif. 
    The radius indicates the deletion ratio within each direction, and the heatmap shows the test--accuracy difference from uniform random edge removal under the same budget. 
    Edges connecting training nodes with the same label are excluded from deletion. 
    }
    \label{fig:supervised}
\end{figure}

\paragraph{Moment changes reveal structural sensitivity in supervised learning.} 
The moment change $\Delta\mathbf{m}(e)$ characterizes the structural role of an edge, allowing edges with similar local structure to be grouped together. 
We divide edges into eight moment-change directions, remove different fractions from each group, and compare against uniform random edge removal under the same budget. 
As shown in Figure~\ref{fig:supervised}, removing many edges from $D4$ consistently degrades performance on both datasets. 
This degradation is associated with more isolated nodes, which limits neighborhood aggregation in these homophily graphs. 
These results illustrate how moment changes provide an interpretable way to analyze the sensitivity of supervised learning to different edge structures.

\paragraph{Moment-guided sampling as graph augmentation.} 

\begin{wraptable}{r}{0.46\textwidth}
\vspace{-1em}
\centering
\caption{Unsupervised node classification accuracy (\%) on Citeseer and Cora. Results are reported as mean $\pm$ standard deviation.}
\label{tab:unsupervised}
\small
\setlength{\tabcolsep}{4pt}
\begin{tabular}{lcc}
\toprule
Method & Citeseer & Cora \\
\midrule
GRACE   & $71.15 \pm 0.73$ & $82.99 \pm 1.53$ \\
MVGRL   & $72.41 \pm 1.30$ & $84.18 \pm 0.98$ \\
GCL-SPAN & $71.94 \pm 0.68$ & $84.07 \pm 0.59$ \\ \midrule
Ours     & $\mathbf{72.65 \pm 0.65}$ & $\mathbf{84.58 \pm 0.56}$ \\
\bottomrule
\end{tabular}
\vspace{-1em}
\end{wraptable}

Graph contrastive learning relies on augmented views that perturb the graph while retaining useful structural information~\citep{zhu2021graph}. 
Since moment-preserving sampling retains several structural properties, we use it to construct an augmented view and compare with GRACE, MVGRL, and GCL-SPAN~\citep{zhu2020deep,hassani2020contrastive,lin2022spectral} on unsupervised node classification. 
As shown in Table~\ref{tab:unsupervised}, our method achieves competitive results on both Citeseer and Cora, despite not being specifically designed for graph contrastive learning. 
Additional experimental details are provided in Appendix~\ref{ap:unsupervised_node_classification}.

%% file: Sections/Related_Work.tex
\section{Related Work}
\textbf{Graph sampling and sparsification.}
Graph sampling and sparsification reduce graph complexity while preserving structural or learning-relevant properties~\citep{chen2023demystifying}. 
Existing methods include random edge dropping for regularization and augmentation~\citep{rong2019dropedge},  subgraph sampling for scalable training~\citep{zeng2019graphsaint}, or graph partitioning to retain community structure~\citep{chiang2019cluster}. 
However, these methods do not explicitly quantify how each edge edit changes a prescribed global structural representation. 
Our method addresses this gap by measuring edge-wise moment changes and selecting edits that move the graph toward a target moment profile.

\textbf{Applications of spectral moments.} 
Spectral moments provide compact summaries of graph structure and diffusion. 
As shown in Appendix~\ref{ap:relationship_diff_moments}, heat trace and personalized PageRank (PPR) trace can be expressed as weighted sums of moments, connecting moments to multi-scale diffusion behavior. 
This connection underlies their use in downstream tasks: 
(1) \textbf{graph representation}, where spectral signatures such as the heat trace summarize local and global structure~\citep{tsitsulin2018netlsd}, and moments provide interpretable graph embeddings~\citep{jin2020spectral};
(2) \textbf{preservation of graph properties}, where moments approximate graph spectra~\citep{cohen2018approximating} and relate to properties such as spectral radius, clustering, and the Estrada index~\citep{preciado2013structural,jin2020spectral,estrada2002characterization}; and 
(3) \textbf{graph learning}, where diffusion operators such as PPR and heat kernels, whose traces are weighted sums of moments, improve graph propagation~\citep{gasteiger2018predict,gasteiger2019diffusion}. 
Unlike these uses of moments as graph-level descriptors or diffusion operators, we use their \emph{changes} as edge-level signals for interpretable and controllable graph sampling.

%% file: Sections/Appendix.tex
\section{Relationships between diffusion kernel and moments}
\label{ap:relationship_diff_moments}

\subsection{Heat Trace: Poisson-weighted Moment Generating Function}
\label{ap:heat_trace}
The normalized heat trace~\citep{tsitsulin2018netlsd} can be seen as the poisson-weighted generating function of moments:
\begin{align}
    \bar{h}(t) = \frac{1}{n} \sum_j e^{-t\lambda_j} &= \frac{1}{n}\text{tr}(e^{-tL})=\frac{1}{n}\text{tr}(e^{-t(I-S)})=e^{-t}\frac{1}{n}\text{tr}(e^{tS}) \\
    &= e^{-t} \sum_{k=0}^{\infty} \frac{t^k}{k!} \frac{1}{n}\text{tr}(S^{k}) \\
    &= e^{-t} \sum_{k=0}^{\infty} \frac{t^k}{k!} m_k 
\end{align}
where $L=I-D^{-\frac{1}{2}}AD^{-\frac{1}{2}}$, $S=D^{-\frac{1}{2}}AD^{-\frac{1}{2}}$, which is similar to $P=D^{-1}A$, and share same eigenvalues $\lambda_i \in [-1, 1]$.
\vspace{2mm}
When $t<<1$, it focus on low-rank, local; when $t>>1$, it focus on high-rank, global, such as connectivity. Smoother than single moment. 

\subsection{PPR Trace: Geometric-weighted Moment Generating Function}
\label{ap:ppr}
Since $P$ is a transition matrix with spectral radius $\rho(P)\leq1$ and $\alpha \in(0,1)$.By definition of PPR and Neumann series, we have 
\begin{equation}
    \Pi_\alpha = \alpha(I - (1-\alpha)P)^{-1} = \alpha \sum_{k=0}^\infty (1-\alpha)^k P^k, 
\end{equation}
Where $\Pi_\alpha$ is pagerank value distribution, and $\alpha$ is restart probability. 
\begin{equation}
    \frac{\operatorname{Tr}(\Pi_\alpha)}{n} = \alpha \sum_{k=0}^\infty (1-\alpha)^k m_k.
\end{equation}
This result demonstrates that moments summarize the global return behavior of diffusion operators, thus preserving moments is related to preserving quantities that underlie PPR diffusion. 

\section{Additional Details for Moments-guided edge sampling}
\label{ap:method_details}

To make our proposed methods in Section~\ref{sec:combinatorial_method},~\ref{sec:algebraic_method} and~\ref{sec:selection} more concrete (e.g., how to update local statistics), we provide more details about the full process in the following sections and pseudocode: 
\begin{center}
\begin{tabular}{ll}
Unified sampling framework
& Section~\ref{ap:sampling_framework}, Algorithm~\ref{alg:moment-guided-sampling} \\
\quad Combinatorial change
& Section~\ref{ap:combinatorial_delta}, Algorithm~\ref{alg:combinatorial-delta} \\
\quad\quad Local-statistic maintenance
& Section~\ref{ap:local_statistics}, Algorithms~\ref{alg:apply-add}--\ref{alg:delta-h} \\
\quad Low-rank change
& Section~\ref{ap:algebric_method}, Algorithms~\ref{alg:lowrank-delta}--\ref{alg:trace-delta} \\
Implementation optimizations
& Section~\ref{ap:efficiency}
\end{tabular}
\end{center}

\subsection{Unified Sampling Framework}
\label{ap:sampling_framework}

In this paper, we propose two exact moment-delta computation methods: the combinatorial closed-walk delta in Section~\ref{sec:combinatorial_method} and the algebraic low-rank delta in Section~\ref{sec:algebraic_method}. 
They serve different needs. 
The combinatorial method gives fast closed-form updates for low-order moments, but is difficult to extend to higher-order moments. 
The algebraic method extends naturally to higher-order moments and batched edits, but has a higher cost. 
Thus, both methods can be used as interchangeable delta-computation modules in a unified sampling framework. 

Algorithm~\ref{alg:moment-guided-sampling} presents this shared framework.
Each candidate edge edit is written as $\epsilon=(o,u,v)$, where $o\in\{\textsc{Add},\textsc{Delete}\}$ specifies the edit type and $(u,v)$ is the node pair to be edited. 
The abstract routine \textsc{DeltaMoment} computes the moment delta for a candidate edit, and can be implemented by either \textsc{CombinatorialDelta} in Algorithm~\ref{alg:combinatorial-delta} or \textsc{LowRankDelta} in Algorithm~\ref{alg:lowrank-delta}.  
After an edit $\epsilon^\ast$ is selected, \textsc{ApplyOperation} applies it to the graph and updates any state required by the chosen delta method. 
For \textsc{CombinatorialDelta}, it calls \textsc{ApplyAdd} or \textsc{ApplyDelete} to maintain $d_i,W_i,M_{ij},H_i$. 
For \textsc{LowRankDelta}, no persistent local statistics are required, so \textsc{ApplyOperation} simply applies $\epsilon^\ast$ to $G$.

\subsection{Combinatorial change for Low-order moments}
\label{ap:combinatorial_delta}

We present \textsc{CombinatorialDelta} in Algorithm~\ref{alg:combinatorial-delta}. 
The edge-deletion case is derived in Section~\ref{sec:combinatorial_method}; here we provide the corresponding formulas for edge addition.

\paragraph{Edge addition induced $\Delta m2$.} 
Adding a new edge $(u, v)\notin E$ introduces one new term in $S_2$, and decreases the weights of existing edges incident to $u$ or $v$, because their degrees increase. 
For node $u$, the degree changes from $d_u$ to $d_u+1$, so the total change over its existing incident edges is 
\[
\sum_{k\in\mathcal{N}(u)}
\left(
\frac{1}{(d_u+1)d_k}
-
\frac{1}{d_ud_k}
\right)
=
-\frac{W_u}{d_u(d_u+1)}.
\]
Combining both endpoints with the new edge contribution gives
\begin{equation}
\Delta S_2^{\mathrm{add}(u,v)}
=
\underbrace{
\frac{1}{(d_u+1)(d_v+1)}
}_{\text{new edge}}
-
\underbrace{
\left[
\frac{W_u}{d_u(d_u+1)}
+
\frac{W_v}{d_v(d_v+1)}
\right]
}_{\text{down-weighted incident edges}}.
\end{equation}

\paragraph{Edge addition induced $\Delta m3$.} 
Adding $(u, v) \notin E$ creates one new triangle for each common neighbor of $u$ and $v$. 
The total mass of these new triangles is $M_{uv}/((d_u+1)(d_v+1))$. 
At the same time, existing triangles incident to $u$ or $v$ are down-weighted because the endpoint degrees increase. 
For node $u$, this change is 
\[
-\frac{1}{d_u(d_u+1)}
\sum_{\{j,k\}\subseteq\mathcal{N}(u),\{j,k\}\in E}
\frac{1}{d_jd_k}
=
-\frac{H_u}{d_u(d_u+1)}.
\]
Thus, 
\begin{equation}
\Delta S_3^{\mathrm{add}(u,v)}
=
\underbrace{
\frac{M_{uv}}{(d_u+1)(d_v+1)}
}_{\text{new triangles}}
-
\underbrace{
\sum_{i\in\{u,v\}}
\frac{H_i}{d_i(d_i+1)}
}_{\text{down-weighted incident triangles}}.
\end{equation}
Finally, the normalized moment deltas are
\[
\Delta m_2=\frac{2}{n}\Delta S_2,
\qquad
\Delta m_3=\frac{6}{n}\Delta S_3.
\]

\subsection{Maintaining Local Statistics}
\label{ap:local_statistics}

The combinatorial method relies on the local statistics $d_i,W_i,M_{ij},H_i$. 
After each edge edit, these quantities must be updated before evaluating the next candidate, so that the moment deltas remain accurate. 
Algorithms~\ref{alg:delta-h},~\ref{alg:apply-add}, and~\ref{alg:apply-delete} give the local update rules for maintaining them without recomputing the whole graph. 

\textsc{DeltaHADD} and \textsc{DeltaHDelete} compute the changes to the triangle masses $H_i$ caused by an edge addition or deletion. 
\textsc{ApplyAdd} and \textsc{ApplyDelete} then update the graph, degrees, neighborhoods, and the maintained statistics $W_i,M_{ij},H_i$. 
All updates are restricted to the edited endpoints and their neighborhoods.

\subsection{Low-rank change for general Edge Edits}
\label{ap:algebric_method}

We represent \textsc{LowRankDelta} in Algorithm~\ref{alg:lowrank-delta}. 
In our experiments, we focus on single-edge deletion. 
For deleting $(u, v)\in E$, assuming $d_u > 1$ and $d_v > 1$, the changed columns are 
\[
p'_u=\frac{A_{:u}-e_v}{d_u-1},
\qquad
p'_v=\frac{A_{:v}-e_u}{d_v-1}.
\]
Thus,
\[
c_u=p'_u-p_u=\frac{p_u-e_v}{d_u-1},
\qquad
c_v=p'_v-p_v=\frac{p_v-e_u}{d_v-1}.
\]
Similarly, for adding a non-edge $(u,v) \notin E$, the changed columns are 
\[
p'_u=\frac{A_{:u}+e_v}{d_u+1},
\qquad
p'_v=\frac{A_{:v}+e_u}{d_v+1},
\]
and therefore
\[
c_u=p'_u-p_u=\frac{e_v-p_u}{d_u+1},
\qquad
c_v=p'_v-p_v=\frac{e_u-p_v}{d_v+1}.
\]
After obtaining the changed columns, we construct $C$ and $R$ and compute $H_t$. 
For a batch $\mathcal{B}$, we first apply all edits locally, then set $c_i=P'_{:i}-P_{:i}$ for each touched endpoint $i\in T$. 
This automatically captures interactions among edited edges through the updated touched columns.

\input{Sections/Appendix_Interaction_Locality}
\subsection{Efficient Implementation}
\label{ap:efficiency}
Moment-guided sampling has two main computational bottlenecks: \textit{evaluating moment changes for individual candidate edits} and \textit{repeatedly rescoring candidates as the graph evolves}. 
We address the first by exploiting the endpoint structure of the low-rank formulation, computing only the transitions needed between touched endpoints in Section~\ref{ap:efficiency_endpoint_transition}. 
For the second, the locality and bounded interactions in Section~\ref{ap:interaction} show that an edge edit affects other candidates primarily within a local region.  
Motivated by this observation, we introduce a lazy top-$k$ rescoring strategy (Section~\ref{ap:lazy_heap}) for both the combinatorial and low-rank methods, reducing the number of exact candidate evaluations while empirically preserving sampling quality (Section~\ref{ap:lazy_results}).

\subsubsection{Endpoint low-rank computation} 
\label{ap:efficiency_endpoint_transition}
In practice, we often use single-edge deletion, i.e., $|\mathcal{B}|=1$, to avoid interactions among multiple simultaneous edits. 
For deleting an edge $(u,v)$, the touched endpoint set is $T=\{u,v\}$ so $r=2$. 
Then each $H_t=R^\top P^tC$ is only a $2\times2$ matrix. 
However, sparse propagation first computes $P^tC\in\mathbb{R}^{n\times2}$ over all $n$ nodes and then keeps only the two rows indexed by $u$ and $v$. 
This costs $O(Km)$ and is wasteful when only four entries of each $H_t$ are needed. 
We therefore compute these entries directly from endpoint-to-endpoint transition probabilities. 

For $a\in\{u,v\}$, one entry of $H_t$ is 
\begin{align}
H_t(a,u)
&=
e_a^\top P^t c_u \\
&=
e_a^\top P^t \frac{Pe_u-e_v}{d_u-1} \\
&=
\frac{e_a^\top P^{t+1}e_u-e_a^\top P^t e_v}{d_u-1}.
\end{align}
Define $s_t(a,b)=e_a^\top P^t e_b$. 
Then
\begin{equation}
H_t(a,u)
=
\frac{s_{t+1}(a,u)-s_t(a,v)}{d_u-1}.
\end{equation}
Similarly, the full matrix can be written as
\begin{equation}
\renewcommand{\arraystretch}{1.4}
H_t =
\begin{bmatrix}
\displaystyle\frac{s_{t+1}(u,u)-s_t(u,v)}{d_u-1}
&
\displaystyle\frac{s_{t+1}(u,v)-s_t(u,u)}{d_v-1}
\\[1.1em]
\displaystyle\frac{s_{t+1}(v,u)-s_t(v,v)}{d_u-1}
&
\displaystyle\frac{s_{t+1}(v,v)-s_t(v,u)}{d_v-1}
\end{bmatrix}.
\end{equation}
If deleting $(u,v)$ isolates an endpoint, e.g., $d_u=1$, we use the direct column change $c_u=-e_v$, so $H_t(a,u)=-s_t(a,v)$. 

Thus, computing $H_t$ reduces to estimating endpoint transition probabilities $s_t(a,b)$ instead of propagating the full sparse product $P^tC$ over the whole graph. 
This only requires propagation within the local neighborhoods of the endpoints. 
The cost becomes 
\[
O\left(\sum_{t=0}^{K}\operatorname{vol}\left(B_t(u)\cup B_t(v)\right)\right),
\]
where $B_t(u)$ is the $t$-hop neighborhood of $u$, and $\operatorname{vol}(\cdot)$ denotes the sum of degrees in the set. 
For low-order moments, this local propagation remains small: to compute $m_k$, we only need transition probabilities up to order $t=k-1$.


\subsubsection{lazy Top-$K$ candidate rescoring}
\label{ap:lazy_heap}
Exhaustively rescoring all candidates after every edge edit is expensive, especially when the sampling budget grows with graph size. 
Motivated by the locality and bounded interactions in Section~\ref{ap:interaction}, we use a lazy top-$K$ strategy that avoids repeatedly evaluating all candidates. 
We first compute the scores of all candidate edges on the original graph and store them in a min-heap. 
At each sampling step, we retrieve the $K$ candidates with the smallest cached scores and recompute only their scores on the current graph. 
We then select the candidate with the smallest updated score, apply the edit, and update its entry in the heap, while the remaining candidates retain their cached scores. 
The selected edge is removed from the heap, and the procedure repeats until the sampling budget is reached. 
This reduces each step from rescoring all $m$ candidates to evaluating only $K\ll m$ candidates, trading exact greedy selection for substantially lower computation. 
We evaluate the resulting sampling quality and runtime in Section~\ref{ap:lazy_results}.

\subsection{Complexity and efficiency evaluation}
\label{ap:lazy_results}
We evaluate the efficiency of moment-guided sampling from both theoretical and empirical perspectives. 
Section~\ref{ap:complexity_analysis} shows that the combinatorial method evaluates low-order moment changes in $O(1)$ time per candidate, while the low-rank method reduces single-edge computation from $O(Kmn)$ to $O(Km)$. 
Section~\ref{ap:lazy_preserve} shows that the heuristic lazy top-$K$ strategy introduced in Section~\ref{ap:lazy_heap} retains comparable property preservation. 
Finally, Section~\ref{ap:runtime_compare} shows that lazy rescoring substantially reduces runtime, achieving over $2\times$ speedup in most settings and up to $6.45\times$. 

\subsubsection{Complexity analysis}
\label{ap:complexity_analysis}
We compare the cost of computing moment changes up to order $K$. 
A direct method in Eq.~\ref{eq:delta_moment} recomputes the matrix powers of the updated transition matrix after each edge removal. 
This costs $O(Kmn)$ time if sparse matrix multiplication is used and the powers become dense, and $O(Kn^3)$ time in the dense worst case. 
It also requires $O(n^2)$ space. 
For the combinatorial method, once the local statistics $d_i,W_i,M_{ij},H_i$ are maintained, computing $\Delta m_2$ and $\Delta m_3$ for one candidate edge only requires constant-time lookups and arithmetic, i.e., $O(1)$ time per candidate. 
Its space cost is $O(n+|\mathcal{P}|)$, where $\mathcal{P}=\{(i,j):\mathcal{N}(i)\cap\mathcal{N}(j)\neq\emptyset\}$ stores pairs with nonzero weighted common-neighbor mass; in the worst case, this becomes $O(n^2)$. 
For the low-rank method, let $T$ be the touched endpoint set and $r=|T|$. 
Constructing the low-rank update $C$ only depends on the old and new neighbors of touched endpoints, costing $\tau=O(\sum_{i\in T}(d_i+d'_i))$. 
Computing all small matrices $H_t=R^\top P^tC$ for $t=0,\ldots,K-1$ by sparse propagation costs $O(Kmr)$. 
After these matrices are obtained, the multiplication between $r\times r$ matrices costs $O(Kr^3)$. 
For fixed moment orders, the terms and coefficients in Eq.~\ref{eq:low_rank_delta} can be precomputed once and reused across candidate edits. 
Thus, $O(\tau+Kmr+Kr^3)\approx O(Kmr+Kr^3)$, and the space cost is $O(m+\tau+Kr^2)$. 

For a single-edge removal $e=(u,v)$, we have $r=2$ and $\tau=O(d_u+d_v)$. 
Per candidate, direct power computation costs $O(Kmn)$ with sparse propagation, or $O(Kn^3)$ in the dense worst case. 
In contrast, the combinatorial method costs $O(1)$, and the low-rank method costs $O(Km)$. 
Their space costs are $O(n^2)$, $O(n+|\mathcal{P}|)$, and $O(m+d_u+d_v+K)$, respectively. 
Thus, the combinatorial method is the fastest for $m_2$ and $m_3$, while the low-rank method is more general and still much faster than direct power computation: its time ratio to direct sparse computation is $(2Km)/(Kmn)=O(1/n)$. 
The $O(Km)$ low-rank cost for single-edge removal can be further optimized in Appendix~\ref{ap:efficiency}. 

\subsubsection{lazy top-$K$ preserves structural properties}  
\label{ap:lazy_preserve}
To evaluate sampling quality of lazy top-$K$ rescoring, we examine the same structural properties as in Section~\ref{sec:graph_properties}.
Figure~\ref{fig:lazy_preservation} compares original moment-guided sampling with lazy top-$K$ variants using different values of $K$.  
Larger $K$ more closely matches original sampling, and $K=25$ already achieves nearly identical property preservation.

\begin{figure}[t]
    \centering
    \begin{minipage}{0.24\linewidth}
        \centering
        \includegraphics[width=\linewidth]{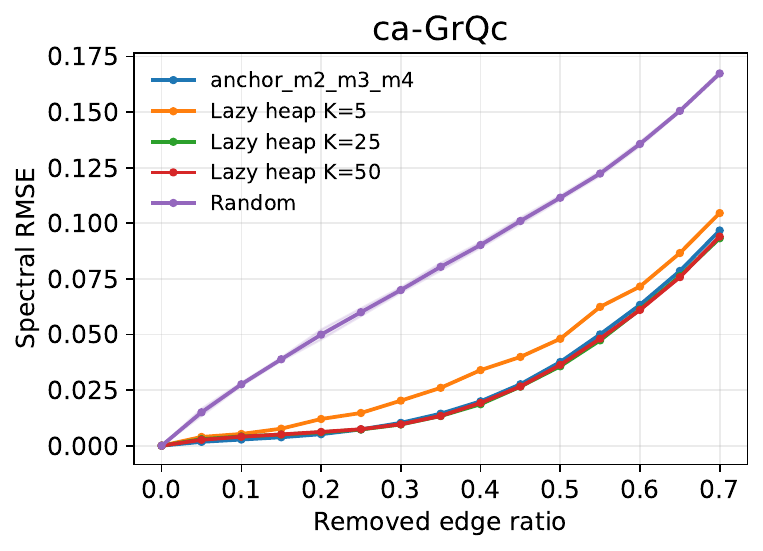}
    \end{minipage}
    \hfill
    \begin{minipage}{0.24\linewidth}
        \centering
        \includegraphics[width=\linewidth]{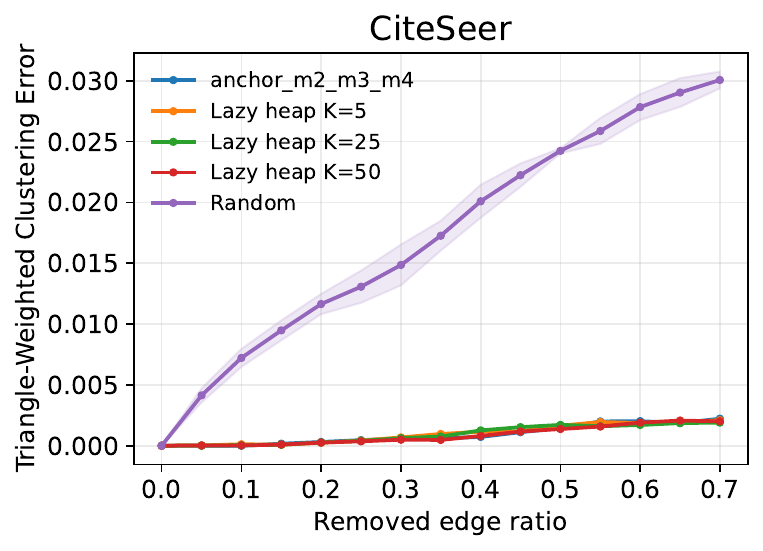}
    \end{minipage}
    \hfill
    \begin{minipage}{0.24\linewidth}
        \centering
        \includegraphics[width=\linewidth]{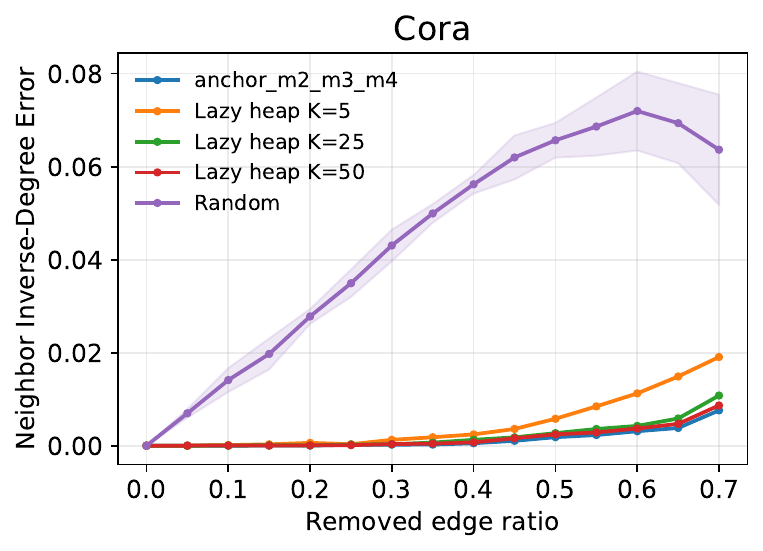}
    \end{minipage}
    \hfill
    \begin{minipage}{0.24\linewidth}
        \centering
        \includegraphics[width=\linewidth]{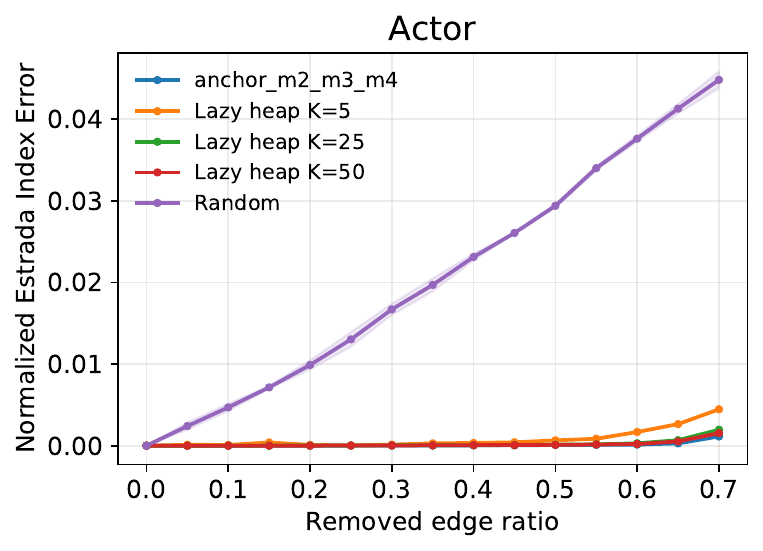}
    \end{minipage}

    \caption{
    Structural-property preservation under original full-rescoring method and lazy top-$K$ candidate rescoring. 
    The \texttt{anchor\_m2\_m3\_m4} curve denotes the original method, using the same moment-preserving setting as in Figure~\ref{fig:preserve_graph_properties}. 
    Each lazy variant preserve the same moment orders but only rescoring only the top-$K$ candidates at each steo. 
    Lazy top-$K$ closely tracks original method, with larger $K$ yielding closer agreement; $K=25$ already nearly matches the original baseline. 
    Random edge removal produces substantially deviations. 
    }
    \label{fig:lazy_preservation}
\end{figure}

\subsubsection{Runtime comparison} 
\label{ap:runtime_compare}

We compare the preprocessing and per-deletion runtime of direct trace, combinatorial, low-rank, and hybrid moment-change computation, together with their lazy top-$K$ variants.
Table~\ref{tab:sampling_efficiency} reports the results on Cora using a single CPU thread. 
The results show four main trends: 
\begin{enumerate}
    \item Both the combinatorial and low-rank methods are substantially faster than direct trace recomputation. 
    \item For low-order moments, the combinatorial method is much faster than low-rank method. 
    For $(m_2,m_3)$, it is nearly $50\times$ faster because maintained local statistics enable constant-time candidate evaluation. 
    \item The hybrid method does not improve runtime over pure low-rank computation for $(m_2,m_3,m_4)$ and is slightly slower in our measurements. 
    This is because the low-rank cost is dominated by the highest moment order: $m_4$ already requires the necessary $H_t$ matrices, while obtaining lower-order moments from them adds little additional cost. 
    \item Lazy top-$K$ rescoring further reduces runtime, achieving $1.62\times$--$6.45\times$ speedup over full rescoring, with smaller $K$ consistently yielding greater speedup. 
\end{enumerate}


\begin{table}[t]
\centering
\caption{
Sampling efficiency on Cora using a single CPU thread. 
\textit{Direct trace} computes edge-wise moment changes by directly evaluating the trace difference in Eq.~\ref{eq:delta_moment}; due to its high cost, it is measured over two sequential deletions, while the other methods are evaluated over 100 deletion. 
Per-deletion time excludes preprocessing. 
\textit{Hybrid} uses the combinatorial method for $m_2,m_3$ and the low-rank method for $m_4$.
A dash in the lazt top-$K$ column denotes the original full-rescoring method, while a numeric $K$ denotes lazy top-$K$ rescoring. 
Speedup is measured relative to the corresponding full-rescoring configuration using the same backend and moment orders.} 
\label{tab:sampling_efficiency}
\small
\setlength{\tabcolsep}{5pt}
\begin{tabular}{lllrrr}
\toprule
\multirow{2}{*}{Moment orders}
& \multirow{2}{*}{Backend}
& \multirow{2}{*}{Lazy top-$K$}
& \multicolumn{2}{c}{Time cost (s)}
& \multirow{2}{*}{Speedup} \\
\cmidrule(lr){4-5}
& & & Preprocessing & Per deletion & \\
\midrule
\multirow{9}{*}{$(m_2,m_3)$}
& Direct trace & -- & 0.197 & \slow{50.16513} & -- \\
\cmidrule(lr){2-6}
& \multirow{4}{*}{Combinatory}
& -- & 0.100 & 0.00067 & -- \\
& & 5  & 0.101 & \fast{0.00028} & \fast{$2.37\times$} \\
& & 25 & 0.098 & 0.00036 & $1.86\times$ \\
& & 50 & 0.100 & 0.00041 & $1.62\times$ \\
\cmidrule(lr){2-6}
& \multirow{4}{*}{Low-rank}
& -- & 0.098 & 0.03192 & -- \\
& & 5  & 0.133 & \fast{0.00728} & \fast{$4.38\times$} \\
& & 25 & 0.133 & 0.00760 & $4.20\times$ \\
& & 50 & 0.132 & 0.00793 & $4.02\times$ \\
\midrule
\multirow{9}{*}{$(m_2,m_3,m_4)$}
& Direct trace & -- & 0.104 & \slow{220.91917} & -- \\
\cmidrule(lr){2-6}
& \multirow{4}{*}{Hybrid}
& -- & 0.107 & 0.04942 & -- \\
& & 5  & 0.245 & \fast{0.00767} & \fast{$6.45\times$} \\
& & 25 & 0.170 & 0.00832 & $5.94\times$ \\
& & 50 & 0.165 & 0.00911 & $5.42\times$ \\
\cmidrule(lr){2-6}
& \multirow{4}{*}{Low-rank}
& -- & 0.101 & 0.04895 & -- \\
& & 5  & 0.154 & \fast{0.00759} & \fast{$6.45\times$} \\
& & 25 & 0.153 & 0.00813 & $6.02\times$ \\
& & 50 & 0.154 & 0.00889 & $5.51\times$ \\
\bottomrule
\end{tabular}
\end{table}

\section{Experimental details}
\label{ap:details_experiments}
This section provides additional experimental details. Appendix~\ref{ap:edge_moment_direction} describes the construction of the visualizations in Figures~\ref{fig:edge_moment_shifts} and~\ref{fig:fingerprint_graphs}. Appendix~\ref{ap:preserve_graph_properties} derives the relationships between spectral moments and the graph properties introduced in Section~\ref{sec:graph_properties}. Appendix~\ref{ap:graph_learning} provides the experimental settings for the graph learning tasks.

\subsection{Edge Moment Direction Profiles}
\label{ap:edge_moment_direction}

\paragraph{Edge-deletion moment vectors.} 
Let $\mathcal{K}=\{k_1,\ldots, k_p\}$ be an ordered collection of $p\ge2$ selected moment orders, defining the moment representation 
\begin{equation}
\mathbf m_{\mathcal K}(G)
=
\begin{bmatrix}
m_{k_1}(G)\\
\vdots\\
m_{k_p}(G)
\end{bmatrix}
\in\mathbb R^p. 
\end{equation}
For each edge $e \in E$, let $G^{-e}=(V,E\setminus\{e\})$ denote the graph obtained by deleting only $e$. 
The deletion-induced change in the $k$-th moment is 
\begin{equation}
\Delta m_k(e)=m_k(G^{-e})-m_k(G). 
\end{equation}
All edges are evaluated independently relative to the same original graph $G$, providing a common reference for their structural effects. 
The \textit{edge-deletion moment vector} is
\begin{equation}
\boldsymbol{\Delta}_e
=
\mathbf m_{\mathcal K}(G^{-e})-\mathbf m_{\mathcal K}(G)
=
\begin{bmatrix}
\Delta m_{k_1}(e)\\
\vdots\\
\Delta m_{k_p}(e)
\end{bmatrix}. 
\end{equation}
These moment changes can be computed using the combinatorial formulas for the supported orders (Section~\ref{sec:combinatorial_method}) or the algebraic method for arbitrary orders (Section~\ref{sec:algebraic_method}). 

\paragraph{Normalization.} 
Because the selected moments may have different numerical scales, we normalize each coordinate by the corresponding moment of the original graph. 
For this relative normalization, we assume $m_k(G)>0$ for every $k\in \mathcal{K}$. 
The \textit{normalized edge-deletion moment vector} is 
\begin{equation}
\widetilde{\boldsymbol{\Delta}}_e
=
\begin{bmatrix}
\Delta m_{k_1}(e)/m_{k_1}(G)\\
\vdots\\
\Delta m_{k_p}(e)/m_{k_p}(G)
\end{bmatrix}.
\end{equation}
This normalized expresses each coordinate as a relative change from the original graph and prevents either moment from dominating solely because of its numerical scale. 

\paragraph{Moment direction and magnitude.} 
We decompose each nonzero normalized vector into its magnitude and unit direction. 
Its magnitude is 
\begin{equation}
r_e
=
\left\|\widetilde{\boldsymbol{\Delta}}_e\right\|_2
=
\sqrt{
\sum_{k\in\mathcal K}
\left(\frac{\Delta m_k(e)}{m_k(G)}\right)^2
},
\end{equation}
and, for $r_e>0$, its direction is
\begin{equation}
\mathbf u_e
=
\frac{\widetilde{\boldsymbol{\Delta}}_e}{r_e}
\in\mathbb S^{p-1},
\qquad
\mathbb S^{p-1}
=
\{\mathbf u\in\mathbb R^p:\|\mathbf u\|_2=1\}.
\end{equation}
The magnitude $r_e$ measures the strength of structural change induced by deleting edge $e$, whereas $\mathbf{u_e}$ describes the relative direction of that change. 
The direction space is the unit circle for $p=2$ and the unit sphere or hypersphere for higher dimensions. 
Zero vectors have no defined direction, so the directional profile uses the edge set $E_+=\{e\in E:r_e>0\}$. 
The following density normalization assumes $E_+\ne\varnothing$. 


\paragraph{Directional density estimation.} 
We measure the separation between two unit directions using their spherical geodesic distance,  
\begin{equation}
d_{\mathbb S}(\mathbf u,\mathbf v)
=
\arccos(\mathbf u^\top\mathbf v)
\in[0,\pi].
\end{equation}
To reduce sensitivity to finite-sample noise and discretization boundaries, we smooth the edge directions using a Gaussian kernel based on this distance, 
\begin{equation}
K_h(\mathbf u,\mathbf u_e)
=
\exp\left[
-\frac12
\left(
\frac{d_{\mathbb S}(\mathbf u,\mathbf u_e)}{h}
\right)^2
\right],
\end{equation}
where $h$ is the angular bandwidth. 
We use $h=3^\circ=\pi/60$ in the reported experiments. 

For $p=2$, write $\mathbf u(\theta)=(\cos\theta,\sin\theta)^\top$.
Then
\begin{equation}
d_{\mathbb S}(\mathbf u(\theta),\mathbf u(\phi))
=
\left|
\bigl((\theta-\phi+\pi)\bmod2\pi\bigr)-\pi
\right|.
\end{equation}
Since the kernel uses the squared distance, this definition reduces exactly to the circular Gaussian kernel in two dimensions. 

Assigning equal weight to every edge in $E_+$, we define the unnormalized directional profile as 
\begin{equation}
q_G(\mathbf u)
=
\sum_{e\in E_+}K_h(\mathbf u,\mathbf u_e).
\end{equation}
Its normalized continuous density with respect to spherical surface measure $dS$ is 
\begin{equation}
p_G(\mathbf u)
=
\frac{q_G(\mathbf u)}
{\displaystyle\int_{\mathbb S^{p-1}}q_G(\mathbf v)\,dS(\mathbf v)},
\qquad
\int_{\mathbb S^{p-1}}p_G(\mathbf u)\,dS(\mathbf u)=1.
\end{equation}
Thus, $p_G$ characterizes the distribution of nonzero edge moment directions in $G$. 
For $p=2$, the surface measure becomes $d\theta$, recovering normalization over the unit circle. 

\paragraph{Directional discretization.} 
For numerical evaluation, partition $\mathbb S^{p-1}$ into cells $\{C_j\}_{j=0}^{J-1}$, with representative directions $\mathbf u_j$ and surface measures $a_j=\int_{C_j}dS>0$. 
The normalization integral can be approximated by 
\begin{equation}
\int_{\mathbb S^{p-1}}q_G(\mathbf u)\,dS(\mathbf u)
\approx
\sum_{j=0}^{J-1}a_jq_G(\mathbf u_j).
\end{equation}
We define the normalized discrete weight of cell $j$ as 
\begin{equation}
w_{G,j}
=
\frac{a_jq_G(\mathbf u_j)}
{\displaystyle\sum_{\ell=0}^{J-1}a_\ell q_G(\mathbf u_\ell)},
\qquad
\sum_{j=0}^{J-1}w_{G,j}=1.
\end{equation}
For sufficiently fine cells, this weight approximates the probability mass of the corresponding directional region, 
\begin{equation}
w_{G,j}
\approx
p_G(\mathbf u_j)a_j
\approx
\int_{C_j}p_G(\mathbf u)\,dS(\mathbf u).
\end{equation}
For equal-area cells, the common factor $a_j$ cancels from the discrete normalization.

\paragraph{Area-normalized visualization.} 
For visualization, we specialize to $\mathcal K=\{2,3\}$ and represent the direction space by the unit circle. 
Using $\mathbf u(\theta)=(\cos\theta,\sin\theta)^\top$, we abbreviate $q_G(\theta)=q_G(\mathbf u(\theta))$ and $p_G(\theta)=p_G(\mathbf u(\theta))$. 
We evaluate the profile on a uniform angular grid 
\begin{equation}
\Theta
=
\{\theta_j=j\delta_\theta\}_{j=0}^{J-1},
\qquad
\delta_\theta=\frac{2\pi}{J}.
\end{equation}

We set $\delta_\theta=0.5^\circ=\pi/360$, giving $J=720$ sectors. 
Each sector has angular measure $a_j=\delta_\theta$, so the general discrete weights reduce to 
\begin{equation}
w_{G,j}
=
\frac{q_G(\theta_j)}
{\displaystyle\sum_{\ell=0}^{J-1}q_G(\theta_\ell)},
\qquad
\sum_{j=0}^{J-1}w_{G,j}=1.
\end{equation}
The corresponding quadrature approximation is 
\begin{equation}
\int_0^{2\pi}q_G(\theta)\,d\theta
\approx
\delta_\theta
\sum_{\ell=0}^{J-1}q_G(\theta_\ell).
\end{equation}
For a sufficiently fine grid, the sector weight satisfies 
\begin{equation}
w_{G,j}
\approx
p_G(\theta_j)\delta_\theta
\approx
\int_{\theta_j-\delta_\theta/2}^{\theta_j+\delta_\theta/2}
p_G(\theta)\,d\theta,
\end{equation}
where the density is interpreted periodically at the endpoints of the angular interval. 
We encode each sector's probability mass by its area, choosing the radius $R_{G,j}$ to satisfy 
\begin{equation}
\frac12R_{G,j}^2\delta_\theta=w_{G,j}.
\end{equation}
This gives 
\begin{equation}
R_{G,j}
=
\sqrt{\frac{2w_{G,j}}{\delta_\theta}}.
\label{eq:edge-profile-sector-radius}
\end{equation}
The total visualization area is therefore 
\begin{equation}
\frac12\sum_{j=0}^{J-1}R_{G,j}^2\delta_\theta
=
\sum_{j=0}^{J-1}w_{G,j}
=
1. 
\end{equation}
Hence, all graphs have equal visualization area, and differences in shape reflect their directional distributions rather than their number of edges.

\paragraph{Graph collections.} 
For a collection $\mathcal{G}=\{G_1,\ldots,G_N\}$, such as PROTEINS or ENZYMES, we compute the normalized cell masses $\{w_{G_g,j}\}_{j=0}^{J-1}$ separately for each graph. 
Consequently, every graph contributes equally to the collection-level mean, regardless of its size. 
The mean profile is 
\begin{equation}
\overline w_{\mathcal G,j}
=
\frac1N\sum_{g=1}^Nw_{G_g,j}.
\end{equation}
Variation across graphs is summarized by the pointwise interquartile  interval 
\begin{equation}
\left[
Q_{0.25}\!\left(\{w_{G_g, j}\}_{g=1}^{N}\right), 
Q_{0.75}\!\left(\{w_{G_g, j}\}_{g=1}^{N}\right) 
\right].
\end{equation}
For the two-dimensional visualization, the solid curve is obtained by applying Eq.~\ref{eq:edge-profile-sector-radius} to the mean masses $\overline w_{\mathcal G,j}$. 
The shaded band is obtained by applying the same radius transformation to the lower and upper pointwise quartiles.

\paragraph{Classic graphs.} 
We provide the specifications of the classic graphs used in Figure~\ref{fig:edge_moment_shifts}. 
\begin{itemize}
    \item Complete graph: 20 nodes; 
    \item Cycle: 60 nodes; 
    \item Path: 60 nodes; 
    \item Lollipop graph: 30 nodes with a clique of size 15; 
    \item House graph: 5 nodes; 
    \item Bull graph: 5 nodes;
    \item Friendship graph: 41 nodes with 20 blades, consisting of one center node and 20 triangles; 
    \item Balanced binary tree: 31 nodes, depth 4, and 30 edges; 
\end{itemize}

\subsection{Derivation and evaluation of structural properties} 
\label{ap:preserve_graph_properties}
This section derives and evaluates the relationships between spectral moments and the graph properties introduced in Section~\ref{sec:graph_properties}.

\subsubsection{Derivations of structural properties} 
\paragraph{Triangle-weighted clustering coefficient.} 
For an undirected graph $G=(V,E)$, let $N_i$ denote the neighborhood of node $i$,
$d_i=|N_i|$, and
\[
\Delta_i = \bigl\{\{j,k\}: j,k \in N_i,\ (j,k)\in E \bigr\}
\]
be the set of triangles incident to $i$, represented by the edges among its neighbors. The standard local clustering coefficient is
\[
C_i
=
\frac{2|\Delta_i|}{d_i(d_i-1)},
\]
with $C_i=0$ when $d_i<2$. 


The third spectral moment of the random-walk matrix $P=D^{-1}A$ can be written as
\begin{align}
m_3
&= \frac{1}{n}\operatorname{Tr}(P^3)
 = \frac{1}{n}\sum_i (P^3)_{ii} \notag\\
&= \frac{1}{n}\sum_i \sum_{j,k}
\frac{A_{ij}A_{jk}A_{ki}}{d_i d_j d_k} \notag\\
&= \frac{1}{n}\sum_i
\frac{2}{d_i}
\sum_{\{j,k\}\in \Delta_i}
\frac{1}{d_j d_k} \notag\\
&= \frac{1}{n}\sum_i
\frac{2|\Delta_i|}{d_i}
\mathbb{E}_{\{j,k\}\sim \Delta_i}
\left[
\frac{1}{d_j d_k}
\right] \notag\\
&= \frac{1}{n}\sum_i
C_i(d_i-1)
\mathbb{E}_{\{j,k\}\sim \Delta_i}
\left[
\frac{1}{d_j d_k}
\right].
\end{align}

When the degrees within each local triangle are comparable, i.e., $d_i\approx d_j\approx d_k$ for $\{j,k\}\in\Delta_i$, and the degrees are not too small, we have
\begin{equation}
C_i(d_i-1)
\mathbb{E}_{\{j,k\}\sim \Delta_i}
\left[
\frac{1}{d_j d_k}
\right]
\approx
C_i
\mathbb{E}_{\{j,k\}\sim \Delta_i}
\left[
\frac{1}{\sqrt{d_j d_k}}
\right]
=
\widetilde C_i.
\end{equation}
Thus, the proposed score can be viewed as a moment-aligned weighted local clustering statistic that assigns larger weights to triangles involving low-degree neighbors. 

We use mean absolute error to measure: 
\begin{equation}
E_{\mathrm{clust}}(G',G)
=
\left|
\frac{1}{|V'|}\sum_{i\in V'} \widetilde C_i(G')
-
\frac{1}{|V|}\sum_{i\in V} \widetilde C_i(G)
\right|.
\end{equation}
where $G$ is the original graph, and $G'$ is the sampled graph.

\paragraph{Degree structure.} 
The second spectral moment is
\begin{equation}
\begin{aligned}
m_2
&= \frac{1}{n}\operatorname{Tr}(P^2) \\
&= \frac{1}{n}\sum_i \sum_{j\in N(i)}
\frac{1}{d_i d_j} \\
&= \frac{1}{n}\sum_i
\mathbb{E}_{j\sim N(i)}
\left[
\frac{1}{d_j}
\right].
\end{aligned}
\end{equation}
Thus, $m_2$ measures the average inverse degree observed from the local
neighborhood of each node.

More generally, let $X_0=i$ be the starting node and $X_r$ be the node reached after $r$ random-walk steps. We define
\begin{equation}
L_i^{(r)}
=
\mathbb{E}
\left[
\frac{1}{d_{X_r}}
\,\middle|\,
X_0=i
\right]
=
\sum_j (P^r)_{ji}\frac{1}{d_j}.
\end{equation}
The even-order spectral moment can be written as
\begin{equation}
\begin{aligned}
m_{2r}
&= \frac{1}{n}\operatorname{Tr}(P^{2r}) \\
&= \frac{1}{n}\sum_i\sum_j
(P^r)_{ij}(P^r)_{ji} \\
&= \frac{1}{n}\sum_i
d_i \sum_j
\frac{(P^r)_{ji}^2}{d_j} \\
&= \frac{1}{n}\sum_i \widetilde L_i^{(r)},
\end{aligned}
\label{eq:m_2r}
\end{equation}
where
\begin{equation}
\widetilde L_i^{(r)}
=
d_i \sum_j
\frac{(P^r)_{ji}^2}{d_j}.
\end{equation}
Compared with $L_i^{(r)}$, the quantity $\widetilde L_i^{(r)}$ reweights each endpoint contribution by the local accessibility factor $d_i(P^r)_{ji}$:
\begin{equation}
\widetilde L_i^{(r)}
=
\sum_j
\left[d_i(P^r)_{ji}\right]
\frac{(P^r)_{ji}}{d_j}.
\label{eq:tilde_L}
\end{equation}
For an $r$-step walk $\pi:i\overset{r}{\leadsto}j$, this factor admits the path expansion
\begin{equation}
d_i(P^r)_{ji}
=
\sum_{\pi:i\overset{r}{\leadsto}j}
\prod_{v\in \operatorname{Int}(\pi)}
\frac{1}{d_v},
\label{eq:accessibility}
\end{equation}
where $\operatorname{Int}(\pi)$ denotes the internal vertices of the walk. 
For example, when $r=2$,
\begin{equation}
d_i(P^2)_{ji}
=
\sum_{k\in N(i)\cap N(j)}
\frac{1}{d_k}.
\end{equation}

Therefore, combining Eq.~\ref{eq:accessibility}, Eq.~\ref{eq:m_2r} and Eq.~\ref{eq:tilde_L}, we can obtain Eq.~\ref{eq:degree_structure}:
\begin{equation}
m_{2r}
=
\frac{1}{n}\sum_i
\mathbb{E}\left[
\left.
\frac{a_i^{(r)}(X_r)}{d_{X_r}}
\right|X_0=i
\right], \quad
a_i^{(r)}(j)
=
\sum_{\pi:i\overset{r}{\leadsto}j}
\prod_{v\in\mathrm{Int}(\pi)}\frac{1}{d_v}.
\end{equation}

\subsubsection{Additional preservation results} 
Figure~\ref{fig:exp2_across_datasets} provides additional results on ca-GrQc~\citep{leskovec2007graph}, Cora and Citeseer~\citep{yang2016revisiting}, and Wisconsin and Actor~\citep{pei2020geom}. 
Across these datasets, moment-preserving sampling consistently preserves the structural properties better than uniform random edge removal.

\begin{figure}[t]
    \centering

    \begin{subfigure}{0.24\linewidth}
        \centering
        \includegraphics[width=\linewidth]{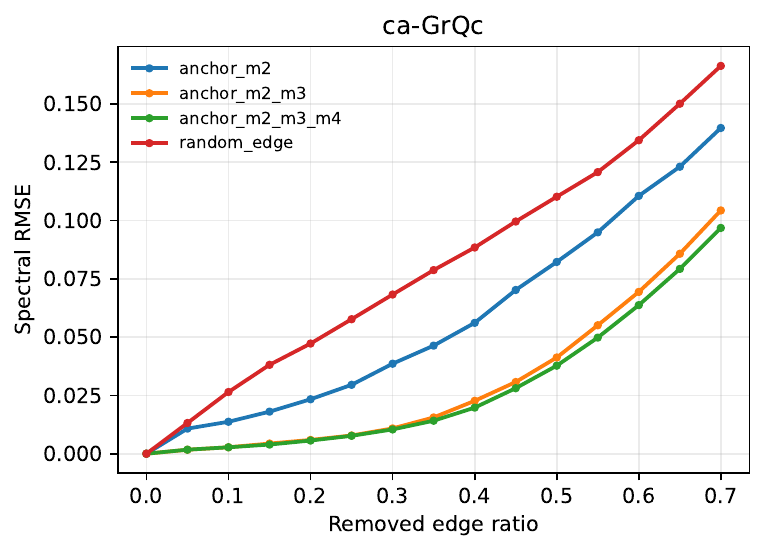}
    \end{subfigure}
    \begin{subfigure}{0.24\linewidth}
        \centering
        \includegraphics[width=\linewidth]{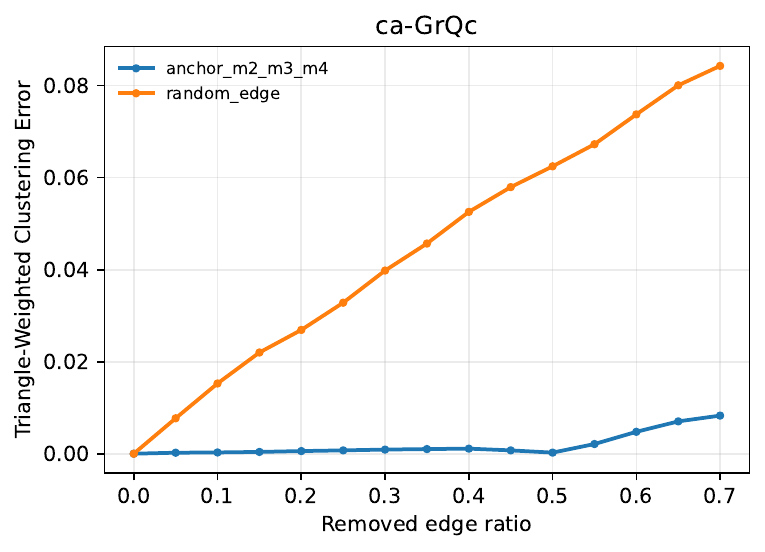}
    \end{subfigure}
    \begin{subfigure}{0.24\linewidth}
        \centering
        \includegraphics[width=\linewidth]{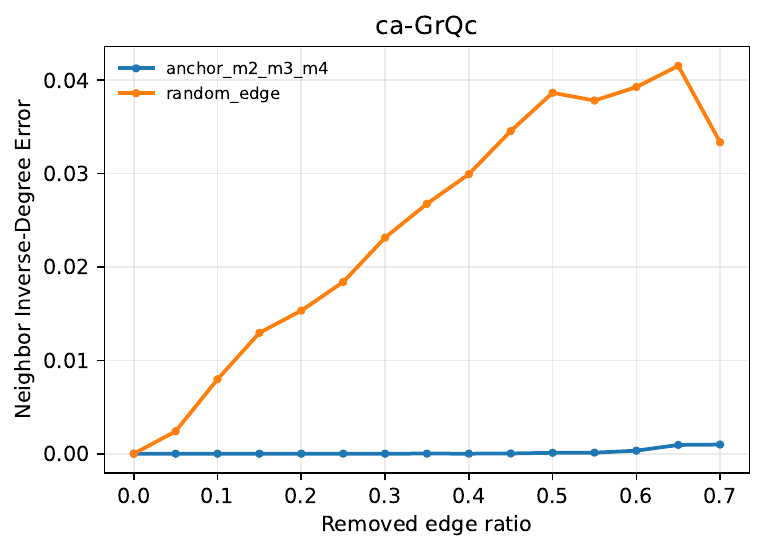}
    \end{subfigure}
    \begin{subfigure}{0.24\linewidth}
        \centering
        \includegraphics[width=\linewidth]{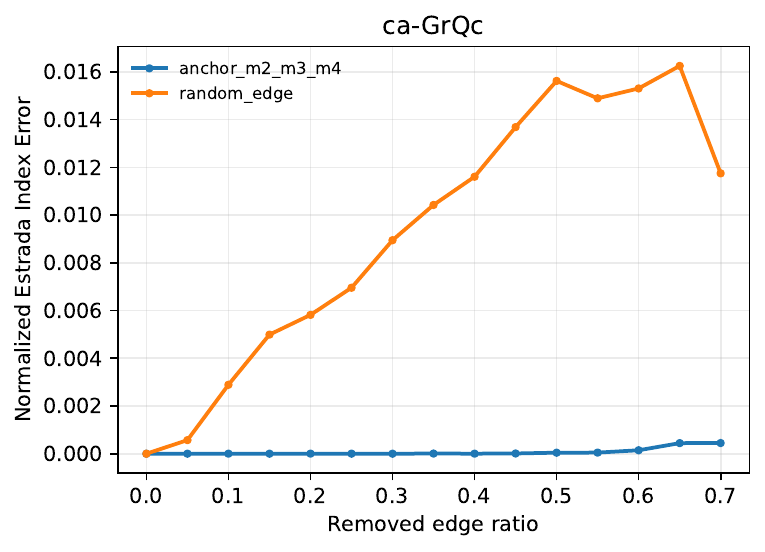}
    \end{subfigure}

    \vspace{1mm}

    \begin{subfigure}{0.24\linewidth}
        \centering
        \includegraphics[width=\linewidth]{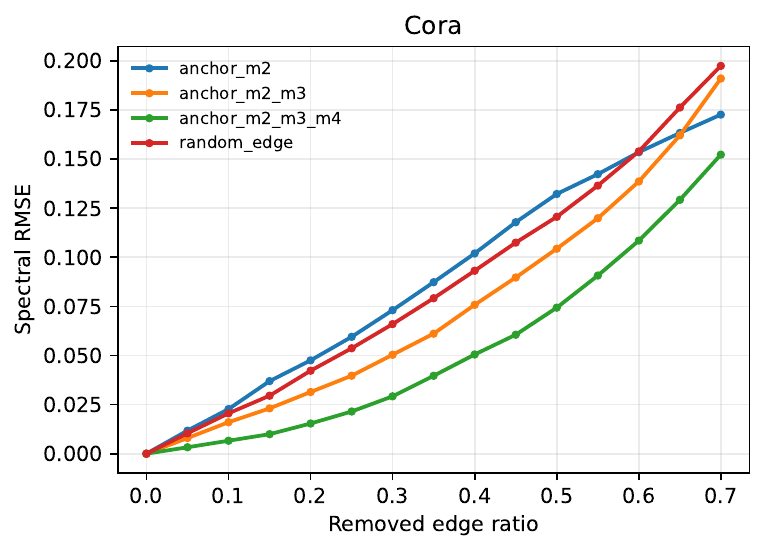}
    \end{subfigure}
    \begin{subfigure}{0.24\linewidth}
        \centering
        \includegraphics[width=\linewidth]{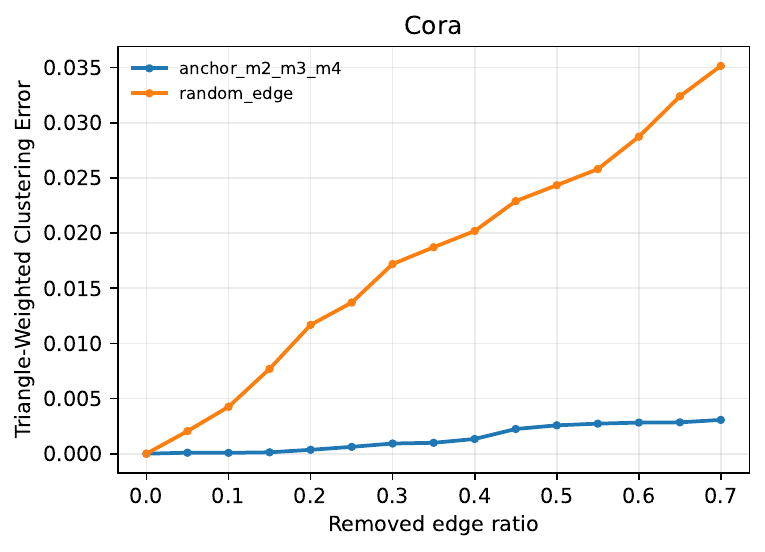}
    \end{subfigure}
    \begin{subfigure}{0.24\linewidth}
        \centering
        \includegraphics[width=\linewidth]{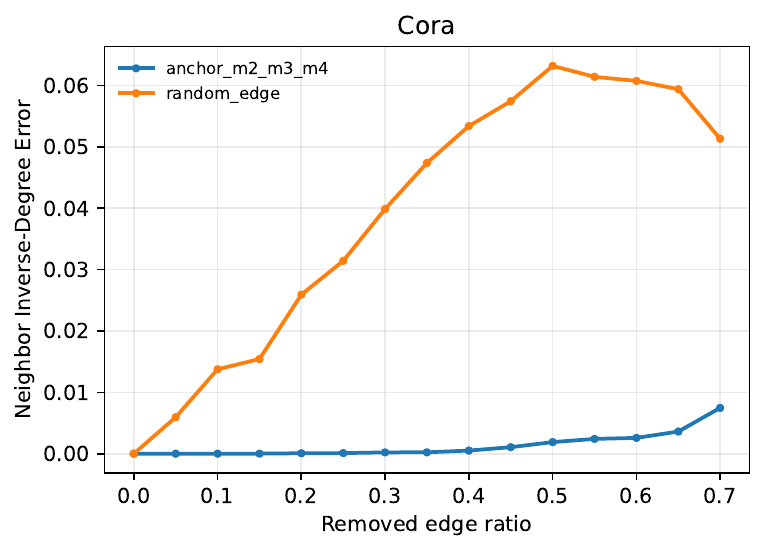}
    \end{subfigure}
    \begin{subfigure}{0.24\linewidth}
        \centering
        \includegraphics[width=\linewidth]{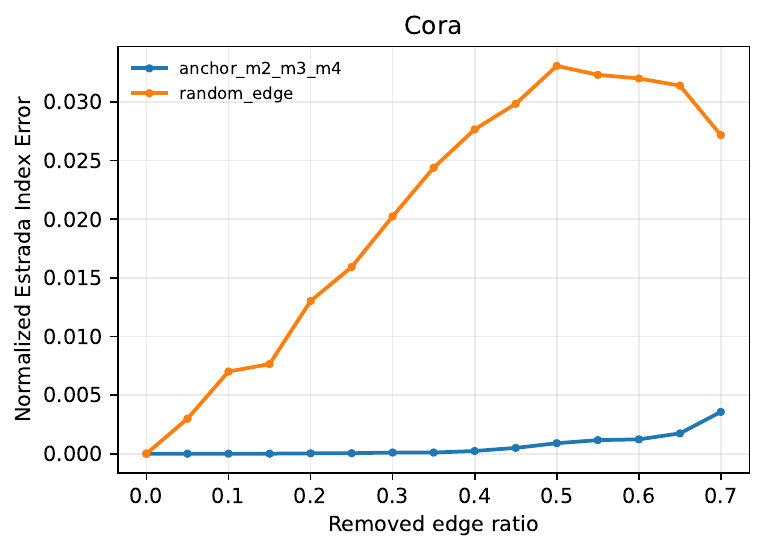}
    \end{subfigure}

    \vspace{1mm}

    \begin{subfigure}{0.24\linewidth}
        \centering
        \includegraphics[width=\linewidth]{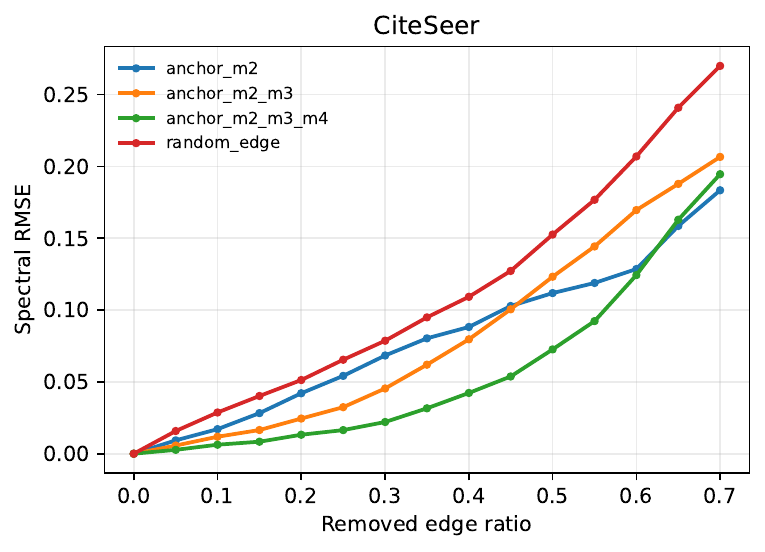}
    \end{subfigure}
    \begin{subfigure}{0.24\linewidth}
        \centering
        \includegraphics[width=\linewidth]{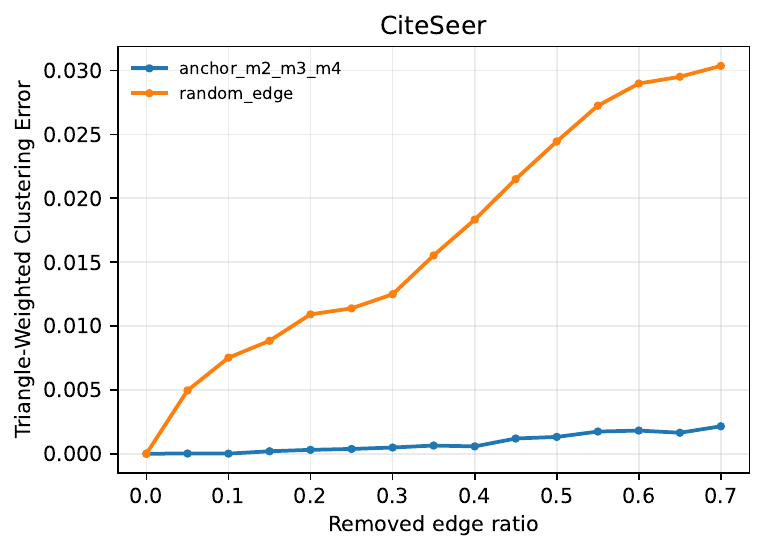}
    \end{subfigure}
    \begin{subfigure}{0.24\linewidth}
        \centering
        \includegraphics[width=\linewidth]{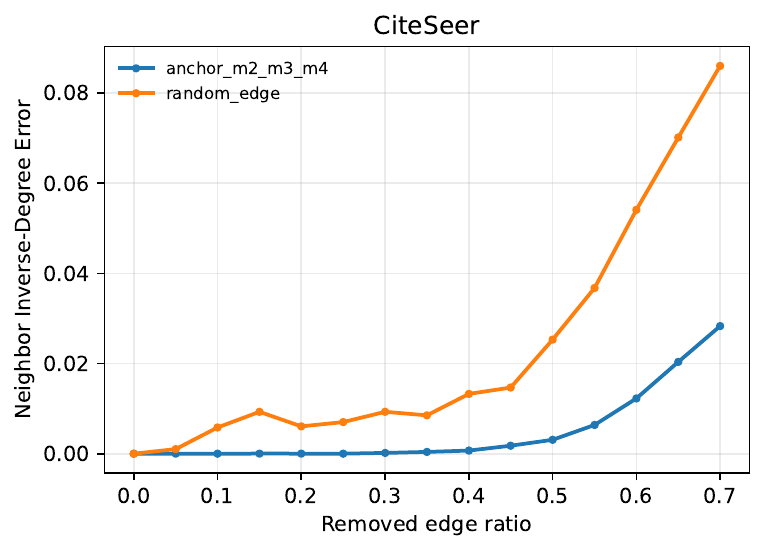}
    \end{subfigure}
    \begin{subfigure}{0.24\linewidth}
        \centering
        \includegraphics[width=\linewidth]{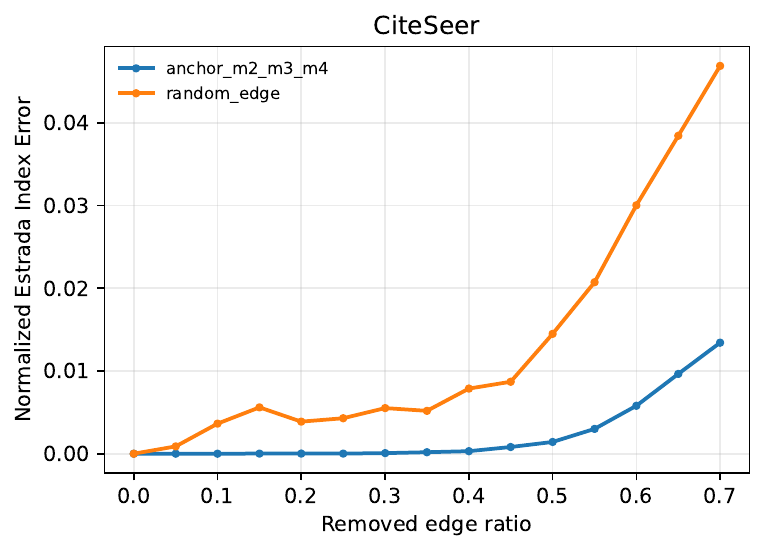}
    \end{subfigure}
    
    \vspace{1mm}

    \begin{subfigure}{0.24\linewidth}
        \centering
        \includegraphics[width=\linewidth]{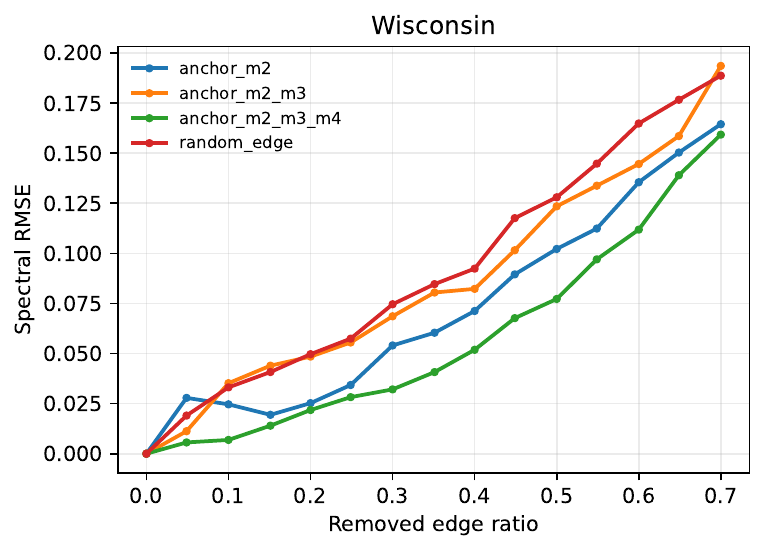}
    \end{subfigure}
    \begin{subfigure}{0.24\linewidth}
        \centering
        \includegraphics[width=\linewidth]{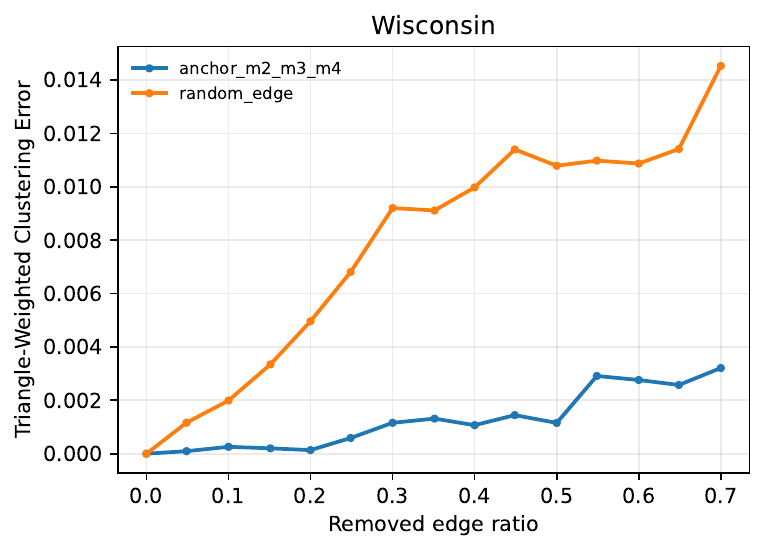}
    \end{subfigure}
    \begin{subfigure}{0.24\linewidth}
        \centering
        \includegraphics[width=\linewidth]{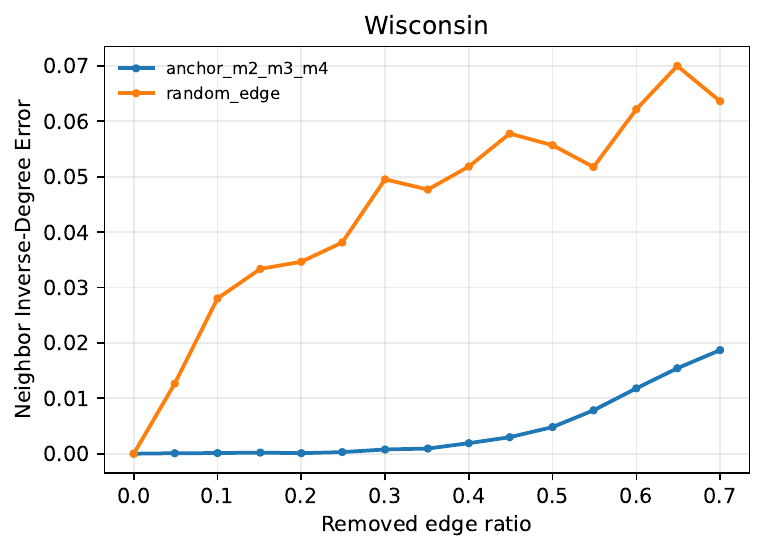}
    \end{subfigure}
    \begin{subfigure}{0.24\linewidth}
        \centering
        \includegraphics[width=\linewidth]{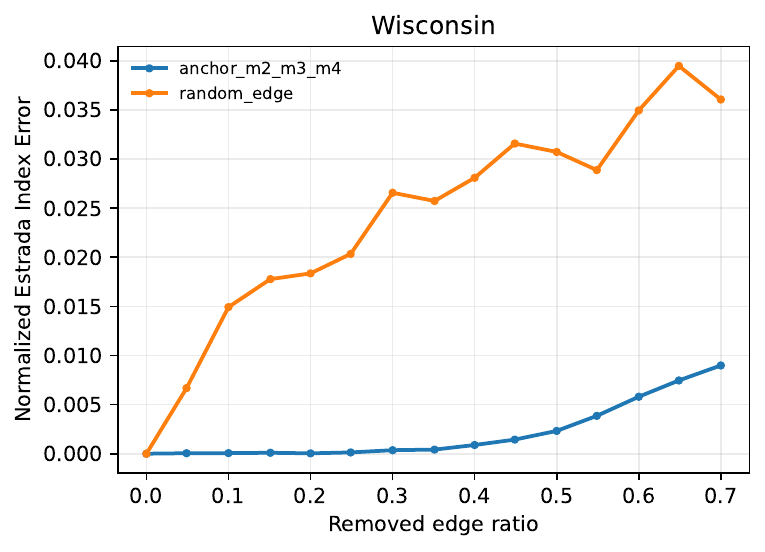}
    \end{subfigure}

    \vspace{1mm}

    \begin{subfigure}{0.24\linewidth}
        \centering
        \includegraphics[width=\linewidth]{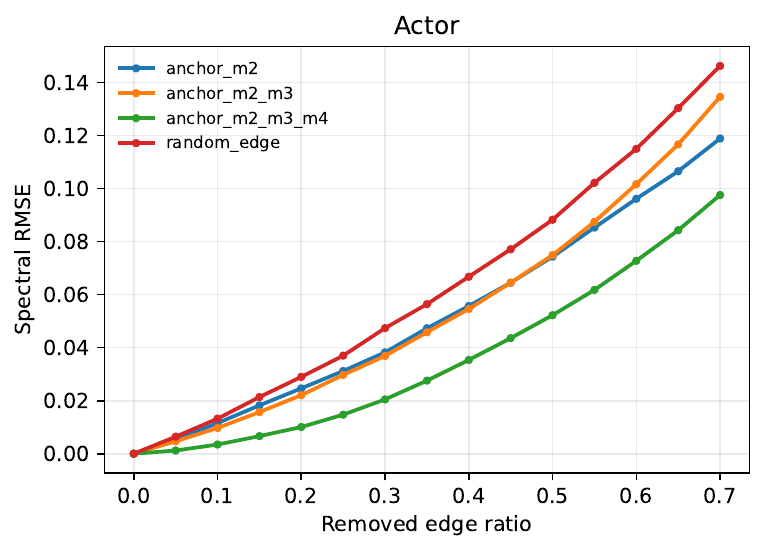}
    \end{subfigure}
    \begin{subfigure}{0.24\linewidth}
        \centering
        \includegraphics[width=\linewidth]{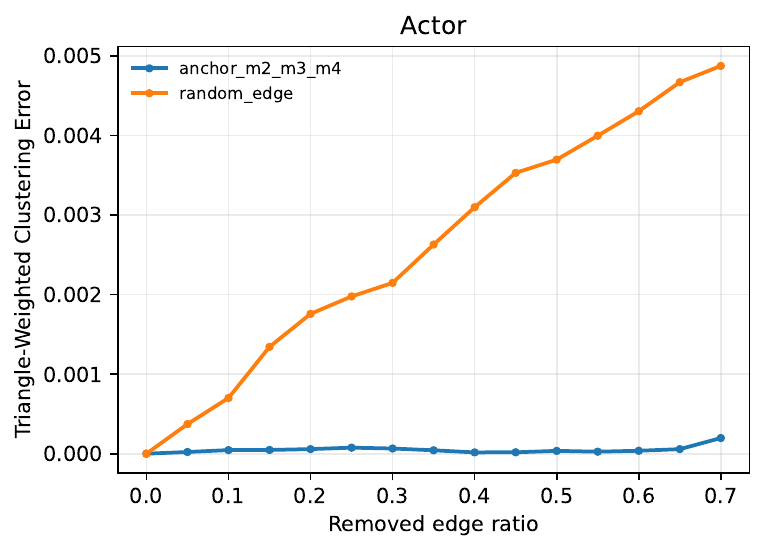}
    \end{subfigure}
    \begin{subfigure}{0.24\linewidth}
        \centering
        \includegraphics[width=\linewidth]{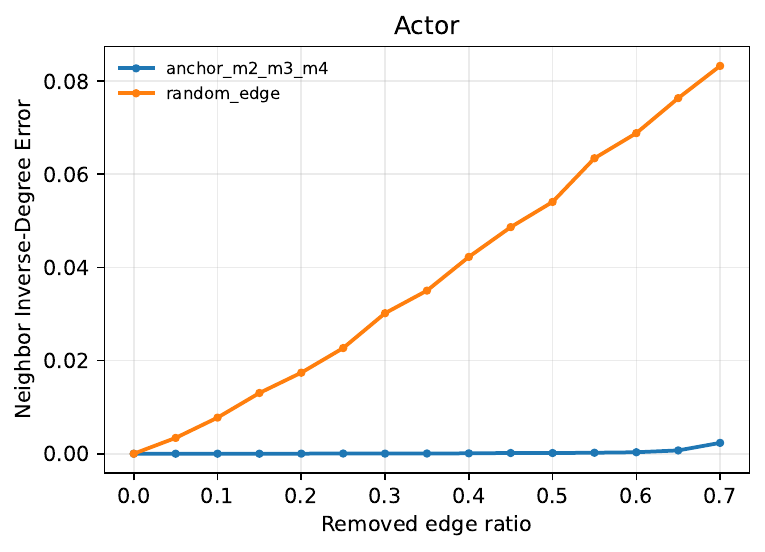}
    \end{subfigure}
    \begin{subfigure}{0.24\linewidth}
        \centering
        \includegraphics[width=\linewidth]{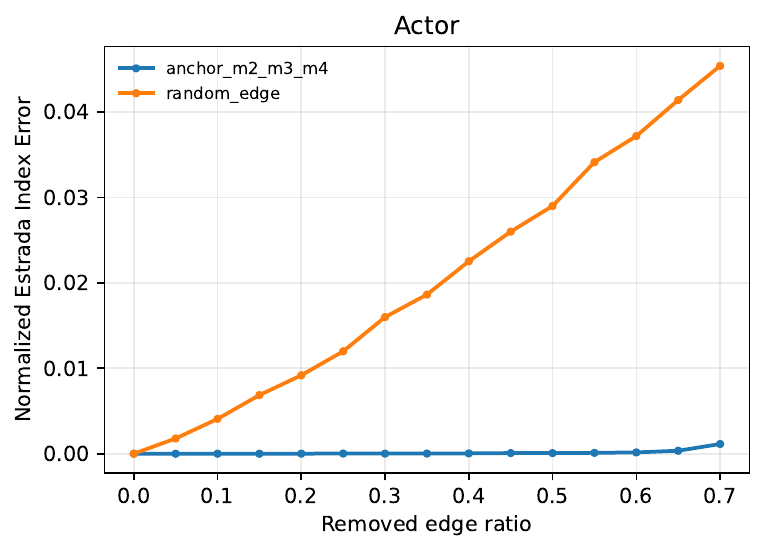}
    \end{subfigure}

    \caption{Moment-preserving sampling preserve four graph properties across datasets.}
    \label{fig:exp2_across_datasets}
\end{figure}

\subsection{Moments-guided edge sampling for graph learning}
\label{ap:graph_learning}
This section provides additional experimental details for Section~\ref{sec:improve_gnns}. 
We first examine how edges with different $\Delta\mathbf{m}(e)$ directions affect supervised node classification, and then evaluate moment-guided augmentation for unsupervised node classification. 

\subsubsection{Supervised node classification}
\label{ap:supervised_node_classification}

\paragraph{Experimental settings.} 
We follow the ``full'' split~\citep{chen2018fastgcn}, where all nodes outside the validation and test sets are used for training. 
During edge removal, we exclude edges whose endpoints are both training nodes with the same label. 
For each (direction, deletion ratio) pair, corresponding to one sector in Figure~\ref{fig:supervised}, we use five random seeds. 
For each seed, we sample five graphs using direction-specific edge removal and five graphs using uniform random removal under the same deletion budget. 
Each direction-specific sample is compared with its random counterpart to obtain five test-accuracy differences per seed. 
We report the mean over all $5\times5=25$ differences as the value of each sector.

\subsubsection{graph contrastive learning}
\label{ap:unsupervised_node_classification}

\paragraph{Experimental settings.} 
For the results in Table~\ref{tab:unsupervised}, we construct a moment-preserving augmented view using an edge-edit budget equal to 20\% of the original number of edges, including both additions and deletions. 
Our implementation is built on the official GCL-SPAN code\footnote{\url{https://github.com/Louise-LuLin/GCL-SPAN/tree/main}}. 
We replace its two augmented views with the original graph and our moment-preserving view, while keeping the remaining training pipeline unchanged. 
Following GCL-SPAN, we use a 10\%/10\%/80\% train/validation/test split for the downstream classifier. 
The released implementation assigns the 80\% split to validation, which we correct to match the stated protocol. 
Since dataset-specific hyperparameters are not fully provided, our reproduced GCL-SPAN results may differ from those reported in the original paper. 
All methods are evaluated over 10 runs using seeds $15$--$24$.

We compare against three representative baselines: 
\begin{itemize}
    \item \textbf{GRACE}~\citep{zhu2020deep} generates two views through random edge dropping. It treats the same node across views as a positive pair and all other nodes as negative pairs. 
    We use the official implementation and hyperparameters\footnote{\url{https://github.com/CRIPAC-DIG/GRACE/tree/master}}. 
    \item \textbf{MVGRL}~\citep{hassani2020contrastive} contrasts representations obtained from the original graph and a diffusion-based view. 
    We use the official implementation and reported hyperparameters\footnote{\url{https://github.com/kavehhassani/mvgrl}}. 
    Its PPR-based augmentation is also closely related to moments: Appendix~\ref{ap:ppr} shows that the trace of PPR can be written as a weighted sum of moments.  
    \item \textbf{GCL-SPAN}~\citep{lin2022spectral}: constructs augmented views by perturbing graph spectrum. 
    Since spectral moments summarize the eigenvalue distribution, its augmentation is closely related to changing moments. 
    We use the official implementation and tune the hyperparameters because dataset-specific settings are not provided. 
    In contrast, targeting individual moments provides a more direct interpretation of which structural properties change, as discussed in Section~\ref{sec:graph_properties} and Appendix~\ref{ap:preserve_graph_properties}. 
\end{itemize}

\paragraph{Searching the moment space for principled graph augmentation.} 
Existing graph contrastive learning methods typically design augmentations through specific heuristics or implicit assumptions about which graph structures should be preserved or perturbed. 
For example, MVGRL relies on diffusion-based views, while GCL-SPAN explicitly perturbs graph spectrum. 
Such designs can be effective, but they do not directly reveal which structural changes are most beneficial for a given graph.

Spectral moments provide an explicit space for studying this question. 
Different target moments correspond to different normalized structural statistics, as discussed in Section~\ref{sec:graph_properties}. 
We therefore use the target moment profile as a controllable augmentation objective and systematically vary it to examine which structural changes lead to better representations. 

Specifically, instead of always setting the target to the original moments, we sample 36 target profiles around the original $(m_2,m_3)$ point. 
For each target, we edit 20\% of the original number of edges and guide the graph toward that target using our moment-guided sampler (Section~\ref{sec:selection}). 
Each target is evaluated over 10 random seeds, $15$--$24$, and we report the mean test accuracy. 
This produces the accuracy landscapes shown in Figure~\ref{fig:accuracy_landscape}. 

The landscapes reveal \textit{graph-specific augmentation preferences rather than a universal optimal rule}. 
For Cora, the best target lies around $(m_2,m_3)=(0.35,0.033)$, where both $m_2$ and $m_3$ increase relative to the original graph.
Since $m_3$ is closely related to triangle-weighted clustering, this direction tends to preserve or strengthen clustered local structure. 
In contrast, Citeseer performs best around $(m_2,m_3)=(0.40,0.0163)$, where both moments decrease. 
The lower $m_3$ suggests that weakening clustered structure can benefit Citeseer, while the opposite trend is preferred by Cora. 
Thus, \textit{different graphs favor different structural perturbations}. 
\textit{The moment landscape makes these preferences explicit by mapping augmentation targets to interpretable structural changes, providing a principled way to search for graph-specific augmentation strategies}.  


\begin{figure}[t]
    \centering
    \begin{subfigure}{0.45\linewidth}
        \centering
        \includegraphics[width=\linewidth]{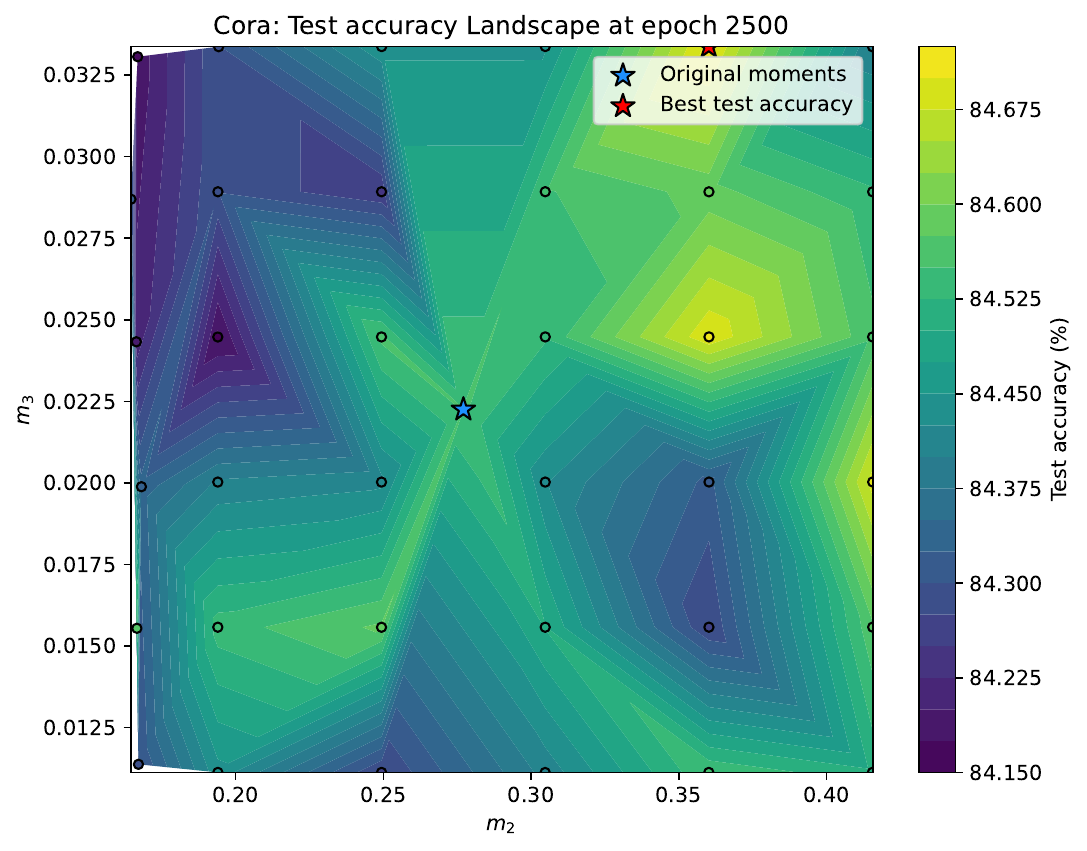}
        \caption{Accuracy landscape on Cora.}
    \end{subfigure}
    \hspace{0.03\linewidth}
    \begin{subfigure}{0.45\linewidth}
        \centering
        \includegraphics[width=\linewidth]{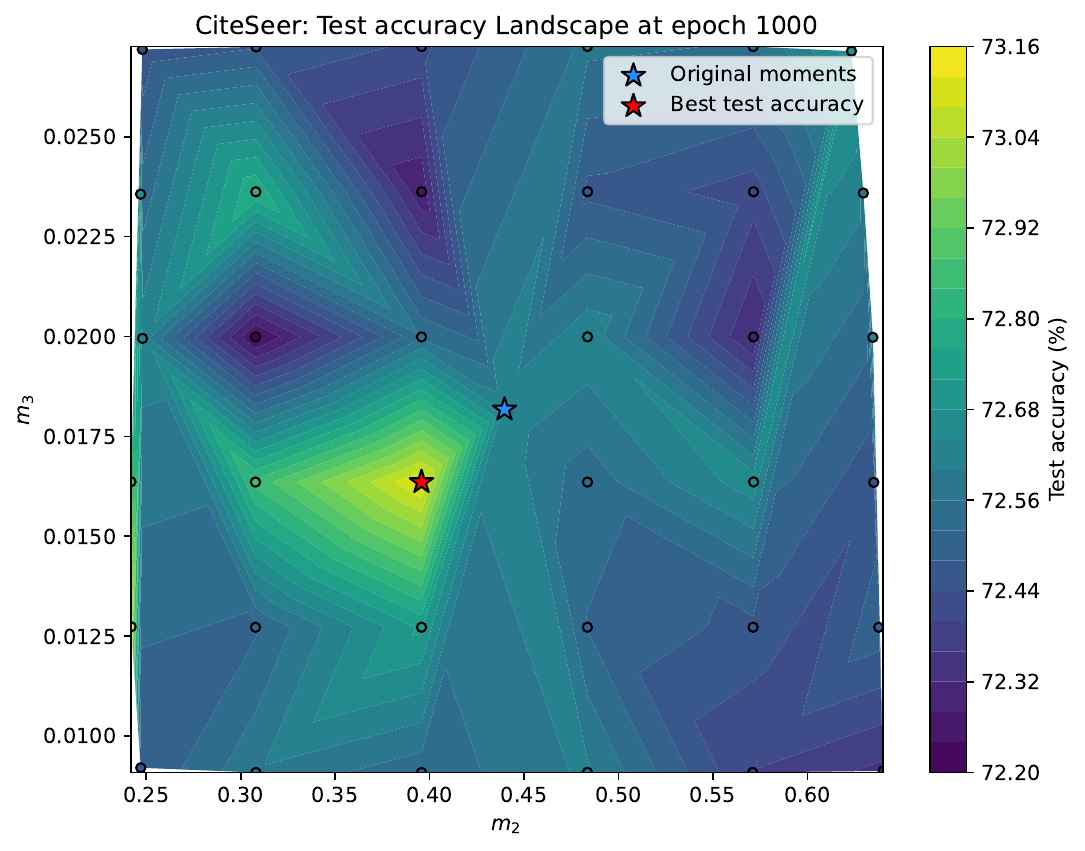}
        \caption{Accuracy landscape on Citeseer.}
    \end{subfigure}
    \caption{Accuracy landscapes over different target moment profiles. 
    Each point corresponds to an augmented view generated toward a target $(m_2,m_3)$ profile under the same edge-edit budget. 
    The blue star marks the original moment profile, while the red star marks the target with the highest mean test accuracy. }
    \label{fig:accuracy_landscape}
\end{figure}

%% file: Sections/Appendix_Interaction_Locality.tex
\subsection{Interactions and Locality of moment changes}
\label{ap:interaction}
This section characterizes how edge edits interact, and how these interactions control the accuracy of composing single-edge moment changes. 
We first establish exact locality

\paragraph{Setting and notation.} 
Let $G=(V,E)$ be a simple, undirected, unweighted graph with a fixed node set of size $n$. 
Write $P_G=A_G D_G^{-1}$, where the column corresponding to an isolated node is defined to be zero. 
Thus $P_G$ is nonnegative and column substochastic, and $\|P_G\|_1\le1$ for the induced matrix $1$-norm. 
No connectivity or minimun-degree constraint is imposed on the resulting graph. 
For an integer $r\ge2$ and a present edge $e$, define
\begin{equation}
    m_k(G)=\frac{1}{n} \operatorname{Tr}(P_G^k),\qquad \Delta_k^G(e)=m_r(G-e) - m_k(G).
    \label{b6:delta}
\end{equation}
For distinct present edges $e,h$, their interaction is the symmetric mixed difference
\begin{equation}
\begin{aligned}
    I_k^G(e,h)
    &=\Delta_k^{G-h}(e) - \Delta_k^G(e)\\
    &=m_k(G-e-h)-m_k(G-e)-m_k(G-h)+m_k(G).
\end{aligned}
\label{b6:interaction}
\end{equation}
When $I_k^G(e,h)=0$, it means remove $e$ or $h$ do not change other's moment $k$ change. 
Let $T_e$ be the endpoint set of $e$. 
We use endpoint-set distance, 
\begin{equation}
    d_G(e,h)=\min_{x\in T_e,\,y\in T_h}d_G(x,y), 
\end{equation}
with distance $+\infty$ between disconnected sets. 
In particular, incident edges have distance zero. 
Unless explicitly stated otherwise, all degrees and neighborhoods in a formula are evaluated in its superscript graph $G$. 

\subsubsection{Exact locality}
\begin{proposition}[Locality of interactions]
For every pair of distinct present edges and every integer $k\ge2$, 
\begin{equation}
    2d_G(e,h)>k\quad\Longrightarrow\quad I_k^G(e,h)=0. 
    \label{b6:locality}
\end{equation}
\end{proposition}

\begin{proof} 
Expand $\operatorname{Tr}(P_G^k)$ as a sum over the length-$k$ closed walks of $G$. 
Use the same walks to index the traces for its subgraphs $G-h$, $G-e$, and $G-e-h$, assigning weight zero to a walk that uses a deleted edge. 
Deleting $e$ only changes transition columns indexed by $T_e$. 
Consequently, a walk that avoids $T_e$ has the same weight in $G$ and $G-e$, and also in $G-h$ and $G-e-h$. 
Its contribution to Eq. \ref{b6:interaction} cancels. 
The same argument applies to walks avoiding $T_h$. 
A potentially nonzero contribution must therefore visit both endpoint sets. 
Such a closed walk contains segment from one set to the other and a segment returning to the first set, each of length at least $d_G(e,h)$. 
Its length is at least $2d_G(e,h)$, proving the claim. 
This argument also covers deletions that create isolated vertices.
\end{proof}

For a collection of orders $\mathcal{K}$ with $R=\max\mathcal{K}$, all selected moment changes of $e$ are unchanged by deleting $h$ whenever $d_G(e,h)>\lfloor R/2\rfloor$. 
For $\mathcal{K}=\{2,3\}$, the locality proposition implies that deleting $h$ leaves the moment changes of $e$ unchanged whenever $d_G(e,h)>1$. 
Thus, edge deletions affect other edges' moment changes only locally; beyond this distance, the moment changes remain unchanged and can be reused without recomputation. 

The distance condition is sufficient but not necessary for zero interaction. 
Edges within the locality radius may also have unchanged moment changes. 
For example, let $G=K_{1,3}$ and choose any two distinct edges $e,h$. 
These edges share the center so $d_G(e,h)=0$. 
With the node set fixed, each of $G$, $G-e$, $G-h$, $G-e-h$ consists of a nonempty star and possibly isolated nodes. 
All four graphs satisfy $m_2=1/2$ and $m_3=0$, giving $I_2^G(e,h)=I_3^G(e,h)=0$. 
Thus, the locality radius identifies a set of potentially affected edges, rather than asserting that every edge within this set is affected.

\begin{corollary}[Additivity of separated deletions]
\label{b6:separated-additivity}
Let $S_1,\ldots,S_p\subseteq E(G)$ be pairwise disjoint, and let $R=\max\mathcal{K}$. 
If $d_G(e,h)>\lfloor R/2\rfloor$ whenever $e\in S_\ell$, $h \in S_j$, and $\ell \ne j$, then, for every $k\in\mathcal{K}$, 
\begin{equation}
    m_k\!\left( G-\bigcup_{\ell=1}^{p}S_\ell \right) - m_k(G)
    = 
    \sum_{\ell=1}^{p} \bigl[m_k(G-S_\ell) - m_k(G)\bigr]. 
\end{equation}
\end{corollary}

\begin{proof}
By the separation condition, no length-$k$ closed walk can visit deleted-edge endpoints from two distinct sets $S_\ell$: such a walk would require length at least $2d_G(e,h)>R\ge k$. 
Thus, each walk's weight is affected by at most one deletion set, and its weight change is additive across the sets. 
Summing over all length-$k$ closed walks, with weight zero for walks using deleted edges, and dividing by $n$ proves the claim. 
\end{proof}

\subsubsection{Interaction bounds for arbitrary moment orders}
\label{b6:interaction_bounds}
The locality result identifies when interaction vanish. 
We now derive an upper bound on $|I_k^G(e,h)|$ for every integer $k\ge2$, using the notation and conventions introduced above. 
The argument exploits the fact that deleting an edge changes only the transition columns indexed by its endpoints. 

\paragraph{Trace inequality.} 
We first record an inequality that accounts for the columns affected by an edge deletion. 
Throughout this section, $\|\cdot\|_1$ denotes the induced matrix $1$-norm for matrices and the vector $1$-norm for column vectors. 

\begin{lemma}[Trace bound using column support]
\label{b6:trace_bound}
Let $B,C\in \mathbf{R}^{n \times n}$. 
If $C_{:i}=0$ for every $i\notin T$, then 
\begin{equation}
\left|\operatorname{Tr}(BC)\right|
\le
\|B\|_1 \sum_{i\in T}\|C_{:i}\|_1. 
\label{eq:trace_bound}
\end{equation}
\end{lemma}

\begin{proof}
Let $\mathbf{e}_i$ denote the $i$th standard basis vector. 
Since columns outside $T$ are zero, 
\[
\operatorname{Tr}(BC)
= 
\sum_{i\in T}\mathbf{e}_i^\top BC_{:i}. 
\]
Consequently, 
\[
\begin{aligned}
\left|\operatorname{Tr}(BC)\right|
&\le
\sum_{i\in T}
\left|\mathbf{e}_i^\top BC_{:i}\right|\\
&\le
\sum_{i\in T}\|BC_{:i}\|_1\\
&\le
\|B\|_1 \sum_{i\in T}\|C_{:i}\|_1, 
\end{aligned}
\]
which proves the claim. 
\end{proof}

\paragraph{Perturbation parameters.} 
For distinct edges $e,h\in E(G)$, define 
\[
P_{00} = P_G, \qquad
P_{10} = P_{G-e}, \qquad
P_{01} = P_{G-h}, \qquad
P_{11} = P_{G-e-h}, 
\]
and let 
\begin{equation}
\begin{aligned}
X&=P_{10} - P_{00}, \\ 
Y&=P_{01} - P_{00}, \\
Z&=P_{11} - P_{10} - P_{01} + P_{00}. 
\end{aligned}
\label{eq:four_corners}
\end{equation}
Thus, $X$ and $X+Z$ describe the transition-matrix changes caused by deleting $e$ before and after deleting $h$, respectively. 
Similarly, $Y$ and $Y+Z$ describe the changes caused by deleting $h$ before and after deleting $e$. 

Both $X$ and $X+Z$ have nonzero columns only in $T_e$, whereas $Y$ and $Y+Z$ have nonzero columns only in $T_h$. 
It follows that $Z$ has nonzero columns only in $T_e\cap T_h$. 

For $i\in T_e$, define 
\[
\rho_{e,i}
= 
\max\left\{
\|X_{:i}\|_1,\, 
\|(X+Z)_{:i}\|_1
\right\}, 
\]
and set 
\begin{equation}
s_e = \sum_{i \in T_e}\rho_{e,i}, 
\qquad
\alpha_e = \max_{i\in T_e}\rho_{e,i}. 
\label{eq:edge_parameters} 
\end{equation}
Define $\rho_{h,i}$, $s_h$, and $\alpha_h$ analogously using $Y$ and $Y+Z$, and let 
\[
\beta = \sum_{i \in V}\|Z_{:i}\|_1. 
\]
Here, $s_e$ bounds the total column change associated with deleting $e$, while $\alpha_e$ bounds its largest column change, accounting for both backgrounds. 
The parameter $\beta$ measures the mixed change in the transition matrix itself. 

\begin{theorem}[Interaction bound for arbitrary moment orders]
\label{thm:uniform_bound}
For every integer $k\ge2$, 
\begin{equation}
\boxed{
|I_k^G(e,h)|
\le
\frac{k\beta + k(k-1)\min\{s_e \alpha_h, s_h\alpha_e\}}{n} 
}.
\label{eq:uniform_bound}
\end{equation}
The bound allows shared endpoints and deletions that create isolated nodes. 
\end{theorem}

\begin{proof}
For $s,t\in[0,1]$, define the bilinear interpolation
\begin{equation}
\begin{aligned}
M(s,t)
&=P_{00} + sX + tY+stZ\\ 
&=(1-s)(1-t)P_{00} + s(1-t)P_{10} + (1-s)tP_{01} + stP_{11}. 
\end{aligned}
\label{eq:interpolation}
\end{equation}
The four coefficients are nonnegative and sum to one. 
Since each $P_{ab}$ is nonnegative and column substochastic, so is $M(s,t)$. 
Hence 
\begin{equation}
\|M(s,t)\|_1 \le 1, 
\qquad
\|M(s,t)^j\|_1 \le 1
\quad
\text{for every integer }j\ge 0.
\label{eq:interpolation_norm} 
\end{equation}
The case $j=0$ uses $M^0=I$ and $\|T\|_1 = 1$. 

Let 
\[
f(s,t) = \frac{1}{n} \operatorname{Tr}(M(s,t)^k). 
\]
At the four corners, $f$ equals the corresponding graph moment. 
Since $f$ is a polynomial in $s$ and $t$, the fundamental theorem of calculus gives 
\begin{equation}
\begin{aligned}
I_k^G(e,h) 
&=f(1,1) - f(1,0) - f(0,1) + f(0,0)\\
&=\int_0^1\int_0^1 f_{st}(s,t)\,dt\,ds. 
\end{aligned}
\label{eq:mixed_integral}
\end{equation}

We suppress the arguments $(s,t)$ below. 
The matrix derivatives are 
\[
M_s=X+tZ,\qquad
M_t=Y+sZ,\qquad
M_{st}=Z. 
\]
By the matrix product rule and cyclic invariance of the trace, 
\[
\begin{aligned}
f_s
&= \frac{1}{n} \sum_{a=0}^{k-1} \operatorname{Tr}(M^a M_s M^{k-1-a})\\ 
&= \frac{k}{n} \operatorname{Tr}(M^{k-1} M_s). 
\end{aligned}
\]
Differentiating once more yields 
\begin{equation}
f_{st}
=\frac{k}{n}\operatorname{Tr}(M^{k-1}Z) +\frac{k}{n}\sum_{j=0}^{k-2} \operatorname{Tr}(M^j M_t M^{k-2-j} M_s). 
\label{eq:mixed_derivative}
\end{equation}
This identity does not require the matrices to commute. 

Since
\[
M_s = (1-t)X + t(X+Z), 
\]
the triangle inequality gives 
\[
\|(M_s)_{:i}\|_1 \le \rho_{e,i} \qquad (i\in T_e), 
\]
and all other columns of $M_s$ are zero. 
Therefore, 
\[
\sum_{i \in T_e}\|(M_s)_{:i}\|_1 \le s_e, \qquad \|M_s\|_1 \le \alpha_e. 
\]
Similarly, 
\[
\sum_{i\in T_h} \|(M_t)_{:i}\|_1 \le s_h, \qquad \|M_t\|_1 \le \alpha_h. 
\]

Applying Lemma~\ref{b6:trace_bound} and Eq.~\ref{eq:interpolation_norm}, we obtain 
\[
\begin{aligned}
\left|\operatorname{Tr}(M^{k-1}Z)\right|
&\le \|M^{k-1}\|_1 \sum_i \|Z_{:i}\|_1
\le \beta,\\
\left|\operatorname{Tr}(M^jM_tM^{k-2-j}M_s)\right|
&\le \|M^jM_tM^{k-2-j}\|_1
\sum_{i\in T_e}\|(M_s)_{:i}\|_1
\le \alpha_h s_e.
\end{aligned}
\]
Alternatively, cyclically rotating the trace gives 
\[
\operatorname{Tr}(M^j M_t M^{k-2-j} M_s)
=
\operatorname{Tr}(M^{k-2-j} M_s M^j M_t), 
\]
which yields the bound $\alpha_e s_h$. 
Thus, each summand is bounded in absolute value by $\min\{s_e\alpha_h, s_h\alpha_e\}$. 

There are $k-1$ summands in Eq.~\ref{eq:mixed_derivative}, so 
\[
|f_{st}(s,t)| \le \frac{k \beta + k(k-1)\min\{s_e\alpha_h, s_h\alpha_e\}}{n}
\]
throughout $[0,1]^2$. 
Taking absolute values in Eq.~\ref{eq:mixed_integral} and integrating this uniform bound proves the result. 

\end{proof}

\paragraph{Explicit degree parameters.} 
The parameters above can be computed from endpoint degrees. 
Deleting an incident edge at a node of degree $d\ge1$ changes its transition column by an amount whose $\ell_1$-norm is 
\begin{equation}
q(d)=
\begin{cases}
1, & d=1, \\
2/d, & d\ge2.
\end{cases}
\label{eq:column_change}
\end{equation}
Indeed, when $d\ge2$, the deleted neighbor loses probability $1/d$, while the remaining $d-1$ neighbors each gain $1/[d(d-1)]$. 
The total absolute change is therefore $2/d$. 
When $d=1$, the column changes from a unit vector to zero, giving norm one. 

If $e$ and $h$ have disjoint endpoint sets, then $Z=0$ and $\beta=0$. 
Writing $d_i=d_G(i)$, we have 
\[
s_e = \sum_{i\in T_e}q(d_i), 
\qquad
\alpha_e = \max_{i\in T_e}q(d_i), 
\]
and analogously for $h$. 
Let 
\[
d_e = \min_{i\in T_e}d_i, 
\qquad
d_h=\min_{i\in T_h}d_i. 
\]
Since $q(d)\le2/d$ for all $d\ge1$, 
\[
s_e \le \frac{4}{d_e}, \quad
\alpha_e \le \frac{2}{d_e}, \quad
s_h \le \frac{4}{d_h}, \quad
\alpha_h \le \frac{2}{d_h}.
\]
Theorem~\ref{thm:uniform_bound} therefore implies
\begin{equation}
\boxed{
|I_k^G(e,h)|
\le
\frac{8k(k-1)}{n\,d_e d_h}.
}
\label{eq:disjoint_bound}
\end{equation}
This bound also covers degree-one endpoints. 

If $e=uv$ and $h=uw$ share an endpoint $u$, then $d_u\ge 2$ because both edges are present. 
Deleting $e$ from $G$ changes column $u$ by $\ell_1$-norm $q(d_u)$, whereas deleting $e$ from $G-h$ changes that column by $\ell_1$-norm $q(d_u-1)$. 
The same two values arise when the roles of $e$ and $h$ are exchanged. 
Since $q(1)=q(2)=1$ and $q(D)=2/D$ for $D\ge 2$, we have 
\[
q(d_u-1)\ge q(d_u),
\qquad
q(d_u-1)=
\begin{cases}
1, &d_u=2,\\
\dfrac{2}{d_u-1}, &d_u\ge3. 
\end{cases}
\]
Thus, the maximum column-change norm at the shared endpoint is exactly $q(d_u-1)$. 
The exact parameter values are therefore 
\begin{equation}
\begin{aligned}
s_e&=q(d_u-1)+q(d_v),
&
\alpha_e&=\max\{q(d_u-1),q(d_v)\},\\
s_h&=q(d_u-1)+q(d_w),
&
\alpha_h&=\max\{q(d_u-1),q(d_w)\},\\
\beta&=
\begin{cases}
1, & d_u=2,\\[2pt]
\dfrac4{d_u(d_u-1)}, & d_u\ge3.
\end{cases}
\end{aligned}
\label{eq:shared_degree_parameters}
\end{equation}
The expressions for $s_e$ and $\alpha_e$ follow because deleting $h$ reduces the degree of $u$ by one and leaves the transition column at $v$ unchanged. 
Similarly, deleting $e$ reduces the degree of $u$ by one and leaves the transition column at $w$ unchanged, giving the expressions for $s_h$ and $\alpha_h$. 
To verify the expression for $\beta$, recall that 
\[
Z=P_{G-e-h}-P_{G-e}-P_{G-h}+P_G,
\qquad
\beta=\sum_i\|Z_{:i}\|_1. 
\]
The matrix $Z$ is supported only on columns in $T_e\cap T_h=\{u\}$, so $\beta=\|Z_{:u}\|_1$. 
For $d_u\ge3$, its entries at rows $v$ and $w$ are 
\[
Z_{vu}=Z_{wu}
=
-\frac1{d_u-1}+\frac1{d_u}
=
-\frac1{d_u(d_u-1)}. 
\]
At each remaining neighbor $x\in N_G(u)\setminus\{v,w\}$, the entry is 
\[
Z_{xu}
=
\frac1{d_u-2}-\frac2{d_u-1}+\frac1{d_u}
=
\frac2{d_u(d_u-1)(d_u-2)}. 
\]
All other entries are zero.
Since there are $d_u-2$ remaining neighbors, summing the absolute values gives
\[
\begin{aligned}
\beta
&=
\frac2{d_u(d_u-1)}
+
(d_u-2)\frac2{d_u(d_u-1)(d_u-2)}\\
&=
\frac4{d_u(d_u-1)}.
\end{aligned}
\]
If $d_u=2$, deleting both edges isolates $u$. 
Under the zero-column convention for isolated nodes, the four relevant columns are 
\[
(P_G)_{:u}=\frac12(\mathbf e_v+\mathbf e_w), 
\qquad
(P_{G-e})_{:u}=\mathbf e_w,
\qquad
(P_{G-h})_{:u}=\mathbf e_v,
\qquad
(P_{G-e-h})_{:u}=0.
\]
Consequently,
\[
Z_{:u}=0-\mathbf e_w-\mathbf e_v+\frac12(\mathbf e_v+\mathbf e_w)
=
-\frac12(\mathbf e_v+\mathbf e_w),
\]
and hence $\beta=\|Z_{:u}\|_1=1$. 

\begin{corollary}[A uniform degree-based bound]
\label{cor:degree_bound}
Suppose that every node in $T_e\cup T_h$ has degree at least $d\ge2$ in $G$. 
Then, for every integer $k\ge2$, 
\begin{equation}
\boxed{
|I_k^G(e,h)|\le 
\frac{1}{n}\left[\frac{4k}{d(d-1)}+\frac{8k(k-1)}{(d-1)^2}
\right]. 
}
\label{eq:common_degree_bound}
\end{equation}
If $T_e\cap T_h=\varnothing$, the sharper bound
\[
|I_k^G(e,h)|\le \frac{8k(k-1)}{nd^2}
\]
holds. 
\end{corollary}

\begin{proof}
Recall that $q(D)\le 2/D$ for every integer $D\ge1$. 
If the edges share an endpoint $u$, then $d_u-1\ge d-1\ge1$, so 
\[
q(d_u-1) \le \frac{2}{d_u-1} \le \frac{2}{d-1}. 
\]
The other endpoint terms satisfy 
\[
q(d_v), q(d_w) \le \frac{2}{d} \le \frac{2}{d-1}. 
\]
Substituting these inequalities into Eq.~\ref{eq:shared_degree_parameters} gives 
\[
s_e, s_h \le \frac{4}{d-1}, 
\qquad
\alpha_e, \alpha_h \le \frac{2}{d-1}. 
\]
If $d_u\ge 3$, then 
\[
\beta = \frac{4}{d_u(d_u-1)} \le \frac{4}{d(d-1)}. 
\]
If $d_u=2$, the condition $2\le d\le d_u$ forces $d=2$, and 
\[
\beta = 1 \le \frac{4}{d(d-1)}. 
\]
Thus, the required bound on $\beta$ also holds when deleting both edges isolates the shared endpoint. 

If the edges have disjoint endpoint sets, then $\beta=0$. 
Each endpoint has degree at least $d$, and the disjoint-endpoint parameter formulas give 
\[
s_e, s_h \le \frac{4}{d} \le \frac{4}{d-1}, 
\qquad
\alpha_e, \alpha_h \le \frac{2}{d} \le \frac{2}{d-1}. 
\]
Therefore, for either endpoint configuration, 
\[
s_e, s_h \le \frac{4}{d-1}, 
\qquad
\alpha_e, \alpha_h \le \frac{2}{d-1}, 
\qquad
\beta \le \frac{4}{d(d-1)}. 
\]
In particular, 
\[
\min\{s_e \alpha_h, s_h \alpha_e\} \le \frac{8}{(d-1)^2}. 
\]
Applying Theorem~\ref{thm:uniform_bound} yields 
\[
\begin{aligned}
|I_k^G(e,h)|
&\le \frac{k\beta + k(k-1) \min\{s_e\alpha_h, s_h\alpha_e\}}{n} \\
&\le \frac{1}{n} \left[ \frac{4k}{d(d-1)} + \frac{8k(k-1)}{(d-1)^2} \right], 
\end{aligned}
\]
which proves Eq.~\ref{eq:common_degree_bound}. 
For disjoint endpoint sets, the stronger estimates $\beta=0$, $s_e, s_h \le 4/d$, and $\alpha_e, \alpha_h \le 2/d$ instead give 
\[
|I_k^G(e,h)| \le \frac{8k(k-1)}{nd^2}. 
\]

\end{proof}

%% file: Sections/pseudocode.tex
\begin{algorithm}[t]
\caption{Moment-Guided Edge Sampling}
\label{alg:moment-guided-sampling}
\begin{algorithmic}[1]
\Require Graph $G=(V,E)$, target moments $\mathbf{m}^\ast$, budget $B$, candidate size $q$, moment orders $\mathcal{K}$, edge-edit delta method $\textsc{DeltaMoment}$, which can be either \textsc{CombinatorialDelta} or \textsc{LowRankDelta}
\Ensure Sampled graph $G_S$
\State Compute current moments $\mathbf{m}=(m_k)_{k\in\mathcal{K}}$  \Comment{Not necessarily continuous, such as $\mathcal{K}={2,5,7}$.}
\For{each $k\in\mathcal{K}$}
    \State $\sigma_k \leftarrow m_k$
    \Comment{Fixed to the original graph moment.}
\EndFor
\For{$t=1,\ldots,B$}
    \State Sample a candidate set $\mathcal{C}_t$ with $|\mathcal{C}_t|=q$, where each $\epsilon=(o,u,v)\in\mathcal{C}_t$ has $o\in\{\textsc{Add},\textsc{Delete}\}$
    \For{each edge edit $\epsilon \in\mathcal{C}_t$}
        \State $\Delta\mathbf{m}(\epsilon)\leftarrow \textsc{DeltaMoment} (G,\epsilon,\mathcal{K})$  
    \EndFor
    \For{each edge edit $\epsilon \in\mathcal{C}_t$}
        \State $\mathrm{score}(\epsilon)\leftarrow \sum_{k\in\mathcal{K}}\left((m_k+\Delta m_k(\epsilon)-m_k^\ast)/\sigma_k\right)^2$
    \EndFor
    \State $\epsilon^\star\leftarrow \arg\min_{e\in\mathcal{C}_t}\mathrm{score}(\epsilon)$
    \State $\Call{ApplyOperation}{G,\epsilon^\star,\textsc{DeltaMoment}}$
    \State $\mathbf{m}\leftarrow \mathbf{m}+\Delta\mathbf{m}(\epsilon^\star)$
\EndFor
\State \Return $G_S\leftarrow G$

\Procedure{ApplyOperation}{$G,\epsilon,\textsc{DeltaMoment}$}
    \State Write $\epsilon=(o,u,v)$
    \If{$\textsc{DeltaMoment}=\textsc{CombinatorialDelta}$} \Comment{Update local statistics}
        \If{$o=\textsc{Add}$}
            \State $\Call{ApplyAdd}{u,v}$
        \Else
            \State $\Call{ApplyDelete}{u,v}$
        \EndIf
    \Else \Comment{Update graph for \textsc{LowRankDelta}}
        \State Apply operation $o$ on pair $(u,v)$ to $G$
    \EndIf
\EndProcedure

\end{algorithmic}
\end{algorithm}

\begin{algorithm}[t]
\caption{\textsc{CombinatorialDelta}: Exact Deltas for $m_2$ and $m_3$}
\label{alg:combinatorial-delta}
\begin{algorithmic}[1]
\Require Graph $G$, candidate edge edit $\epsilon=(o,u,v)$, maintained statistics $d_i,W_i,M_{ij},H_i$ 
\Ensure Exact moment changes $\Delta m_2(\epsilon),\Delta m_3(\epsilon)$
\If{$o=\textsc{Add}$}
    \State $\Delta S_2\gets
    \dfrac{1}{(d_u+1)(d_v+1)}
    -\dfrac{W_u}{d_u(d_u+1)}
    -\dfrac{W_v}{d_v(d_v+1)}$
    \State $\Delta S_3\gets
    \dfrac{M_{uv}}{(d_u+1)(d_v+1)}
    -\dfrac{H_u}{d_u(d_u+1)}
    -\dfrac{H_v}{d_v(d_v+1)}$
\ElsIf{$o=\textsc{Delete}$}
    \State $\Delta S_2\gets
    \dfrac{W_u-d_v^{-1}}{d_u(d_u-1)}
    +\dfrac{W_v-d_u^{-1}}{d_v(d_v-1)}
    -\dfrac{1}{d_ud_v}$
    \State $\Delta S_3\gets
    \dfrac{H_u-M_{uv}/d_v}{d_u(d_u-1)}
    +\dfrac{H_v-M_{uv}/d_u}{d_v(d_v-1)}
    -\dfrac{M_{uv}}{d_ud_v}$
\EndIf
\State \Return $\left(\dfrac{2}{n}\Delta S_2,\dfrac{6}{n}\Delta S_3\right)$






\end{algorithmic}
\end{algorithm}

\begin{algorithm}[t]
\caption{\textsc{ApplyAdd}: Update Local Statistics for \textsc{CombinatorialDelta}}
\label{alg:apply-add}
\begin{algorithmic}[1]
\Procedure{ApplyAdd}{$u,v$}
    \State $N_u\gets\mathcal{N}(u)$, $N_v\gets\mathcal{N}(v)$
    \State Use old degrees $d_u,d_v$ \Comment{Update degrees after finish updating all local statistics}
    \State $\Delta H\gets\Call{DeltaHAdd}{u,v}$

    \ForAll{$x\in N_u$} \Comment{Update the local statistics between $u$'s neighbors and $v$}
        \State $M_{vx},M_{xv}\gets M_{vx}+(d_u+1)^{-1},M_{xv}+(d_u+1)^{-1}$ \Comment{Edge addition introduce a new common neighbor $u$ between $v$ and $u$'s neighbors}
        \State $W_x\gets W_x-\bigl(d_u(d_u+1)\bigr)^{-1}$ \Comment{The increase in $u$'s degree dilutes the weights assigned to its neighbors'}
    \EndFor

    \ForAll{$x\in N_v$} \Comment{Update the local statistics between $v$'s neighbors and $u$}
        \State $M_{ux},M_{xu}\gets M_{ux}+(d_v+1)^{-1},M_{xu}+(d_v+1)^{-1}$ \Comment{Edge addition introduce a new common neighbor $v$ between $u$ and $v$'s neighbors}
        \State $W_x\gets W_x-\bigl(d_v(d_v+1)\bigr)^{-1}$  \Comment{The increase in $v$'s degree dilutes the weights assigned to its neighbors'}
    \EndFor

    \ForAll{$x,y\in N_u$} \Comment{Update the local statistics within $u$'s neighbors}
        \State $M_{xy}\gets M_{xy}-\bigl(d_u(d_u+1)\bigr)^{-1}$  \Comment{The increase in $u$'s degree dilutes the weights assigned to its neighbors'}
    \EndFor
    \ForAll{$x,y\in N_v$} \Comment{Update the local statistics within $v$'s neighbors}
        \State $M_{xy}\gets M_{xy}-\bigl(d_v(d_v+1)\bigr)^{-1}$  \Comment{The increase in $v$'s degree dilutes the weights assigned to its neighbors'}
    \EndFor

    \State Add $\{u,v\}$ to $E$
    \State $\mathcal{N}(u)\gets N_u\cup\{v\}$, $\mathcal{N}(v)\gets N_v\cup\{u\}$
    \State $d_u\gets d_u+1$, $d_v\gets d_v+1$
    \State $W_u\gets W_u+d_v^{-1}$, $W_v\gets W_v+d_u^{-1}$ \Comment{Edge addition introduce a new term for $u$ and $v$, after updating degrees}

    \ForAll{$i$ with $\Delta H_i\neq 0$}
        \State $H_i\gets H_i+\Delta H_i$
    \EndFor
\EndProcedure
\end{algorithmic}
\end{algorithm}

\begin{algorithm}[t]
\caption{\textsc{ApplyDelete}: Update Local Statistics for \textsc{CombinatorialDelta}}
\label{alg:apply-delete}
\begin{algorithmic}[1]
\Procedure{ApplyDelete}{$u,v$}
    \State $N_u\gets\mathcal{N}(u)$, $N_v\gets\mathcal{N}(v)$
    \State Use old degrees $d_u,d_v$
    \State $N_u^-\gets N_u\setminus\{v\}$, $N_v^-\gets N_v\setminus\{u\}$
    \State $\Delta H\gets\Call{DeltaHDelete}{u,v}$

    \ForAll{$x\in N_u^-$} \Comment{Update the local statistics between $u$'s neighbors and $v$}
        \State $M_{vx},M_{xv}\gets M_{vx}-d_u^{-1},M_{xv}-d_u^{-1}$ \Comment{Edge deletion remove the common neighbor $u$ between $v$ and $u$'s neighbors}
    \EndFor
    \ForAll{$x\in N_v^-$} \Comment{Update the local statistics between $v$'s neighbors and $u$}
        \State $M_{ux},M_{xu}\gets M_{ux}-d_v^{-1},M_{xu}-d_v^{-1}$ \Comment{Edge deletion remove the common neighbor $v$ between $u$ and $v$'s neighbors}
    \EndFor

    \ForAll{$x,y\in N_u^-$} \Comment{Update the local statistics within $u$'s neighbors}
        \State $M_{xy}\gets M_{xy}+\bigl(d_u(d_u-1)\bigr)^{-1}$ \Comment{The decrease in $u$'s degree concentrate the weights assigned to its neighbors'}
    \EndFor
    \ForAll{$x,y\in N_v^-$} \Comment{Update the local statistics within $u$'s neighbors}
        \State $M_{xy}\gets M_{xy}+\bigl(d_v(d_v-1)\bigr)^{-1}$ \Comment{The decrease in $v$'s degree concentrate the weights assigned to its neighbors'}
    \EndFor

    \State $W_u\gets W_u-d_v^{-1}$, $W_v\gets W_v-d_u^{-1}$ \Comment{Edge deletion remove one term for $u$ and $v$, before updating degrees}
    \State Delete $\{u,v\}$ from $E$
    \State $\mathcal{N}(u)\gets N_u^-$, $\mathcal{N}(v)\gets N_v^-$
    \State $d_u\gets d_u-1$, $d_v\gets d_v-1$

    \ForAll{$x\in N_u^-$} \Comment{Edge deletion concentrate weights assigned to $u$'s neighbors}
        \State $W_x\gets W_x+\bigl((d_u+1)d_u\bigr)^{-1}$
    \EndFor
    \ForAll{$x\in N_v^-$} \Comment{Edge deletion concentrate weights assigned to $v$'s neighbors}
        \State $W_x\gets W_x+\bigl((d_v+1)d_v\bigr)^{-1}$
    \EndFor

    \ForAll{$i$ with $\Delta H_i\neq 0$}
        \State $H_i\gets H_i+\Delta H_i$
    \EndFor
\EndProcedure
\end{algorithmic}
\end{algorithm}

\begin{algorithm}[t]
\caption{Exact Local Update of $H$}
\label{alg:delta-h}
\begin{algorithmic}[1]
\Function{WCommon}{$a,b,\mathcal{R}$}
    \Comment{Weighted common-neighbor mass, excluding nodes in $\mathcal{R}$}
    \State $s\gets 0$
    \ForAll{$x\in\mathcal{N}(a)\cap\mathcal{N}(b)$}
        \If{$x\notin\mathcal{R}$}
            \State $s\gets s+d_x^{-1}$
        \EndIf
    \EndFor
    \State \Return $s$
\EndFunction

\Function{DeltaHAdd}{$u,v$}
    \State Initialize $\Delta H_i\gets 0$ for
    $i\in\{u,v\}\cup\mathcal{N}(u)\cup\mathcal{N}(v)$
    \State $C_{uv}\gets\Call{WCommon}{u,v,\emptyset}$ \Comment{New triangles centered at $u$ and $v$}
    \State $\Delta H_u\gets \Delta H_u+C_{uv}/(d_v+1)$
    \State $\Delta H_v\gets \Delta H_v+C_{uv}/(d_u+1)$

    \ForAll{$i\in\mathcal{N}(u)$} \Comment{Old triangles containing $u$ diluted because $d_u$ increases}
        \State $C_{iu}\gets\Call{WCommon}{i,u,\emptyset}$
        \State $\Delta H_i\gets
        \Delta H_i-C_{iu}/\bigl(d_u(d_u+1)\bigr)$
    \EndFor

    \ForAll{$i\in\mathcal{N}(v)$} \Comment{Old triangles containing $v$ diluted because $d_v$ increases}
        \State $C_{iv}\gets\Call{WCommon}{i,v,\emptyset}$
        \State $\Delta H_i\gets
        \Delta H_i-C_{iv}/\bigl(d_v(d_v+1)\bigr)$
    \EndFor

    \ForAll{$i\in\mathcal{N}(u)\cap\mathcal{N}(v)$} \Comment{New triangle $(i,u,v)$ for each common neighbor $i$}
        \State $\Delta H_i\gets
        \Delta H_i+\bigl((d_u+1)(d_v+1)\bigr)^{-1}$
    \EndFor

    \State \Return $\Delta H$
\EndFunction

\Function{DeltaHDelete}{$u,v$}
    \State Initialize $\Delta H_i\gets 0$ for
    $i\in\{u,v\}\cup\mathcal{N}(u)\cup\mathcal{N}(v)$
    \State $C_{uv}\gets\Call{WCommon}{u,v,\emptyset}$ \Comment{Triangles centered at $u$ and $v$ using edge $(u,v)$ disappear}

    \State $\Delta H_u\gets\Delta H_u-C_{uv}/d_v$
    \State $\Delta H_v\gets\Delta H_v-C_{uv}/d_u$

    \ForAll{$i\in\mathcal{N}(u)\setminus\{v\}$} \Comment{Remaining old triangles containing $u$ grow because $d_u$ decreases}
        \State $C_{iu}^{-v}\gets\Call{WCommon}{i,u,\{v\}}$
        \State $\Delta H_i\gets
        \Delta H_i+C_{iu}^{-v}/\bigl(d_u(d_u-1)\bigr)$
    \EndFor

    \ForAll{$i\in\mathcal{N}(v)\setminus\{u\}$} \Comment{Remaining old triangles containing $v$ grow because $d_v$ decreases}
        \State $C_{iv}^{-u}\gets\Call{WCommon}{i,v,\{u\}}$
        \State $\Delta H_i\gets
        \Delta H_i+C_{iv}^{-u}/\bigl(d_v(d_v-1)\bigr)$
    \EndFor

    \ForAll{$i\in\mathcal{N}(u)\cap\mathcal{N}(v)$} \Comment{Triangle $(i,u,v)$ disappears for each common neighbor $i$}
        \State $\Delta H_i\gets\Delta H_i-(d_ud_v)^{-1}$
    \EndFor

    \State \Return $\Delta H$
\EndFunction
\end{algorithmic}
\end{algorithm}

\begin{algorithm}[t]
\caption{\textsc{LowRankDelta}: Exact Moment Delta for Batched Edge Edits}
\label{alg:lowrank-delta}
\begin{algorithmic}[1]    
    \Require Graph $G=(V,E)$, transition matrix $P=AD^{-1}$, edit batch $\mathcal{B}$, moment orders $\mathcal{K}$
    \Ensure Exact changes $\{\Delta m_k\}_{k\in\mathcal{K}}$
    \State $T\leftarrow \{i\in V:\exists j\in V \text{ such that } (i,j)\in\mathcal{B}\}$
    \State Write $T=\{i_1,\ldots,i_r\}$
    \State Compute updated touched columns $P'_{:i_a}$ for all $i_a\in T$
    \State Initialize $C\in\mathbb{R}^{n\times r}$ and $R\in\mathbb{R}^{n\times r}$
    \For{$a=1,\ldots,r$}
        \State $C_{:a}\leftarrow P'_{:i_a}-P_{:i_a}$
        \State $R_{:a}\leftarrow e_{i_a}$
    \EndFor
    \State $K_{\max}\leftarrow\max\mathcal{K}$
    \State $Y\leftarrow C$
        \For{$t=0,\ldots,K_{\max}-1$}
        \State $H_t\leftarrow R^\top Y$
        \If{$t<K_{\max}-1$}
            \State $Y\leftarrow PY$
        \EndIf
    \EndFor
    \For{each $k\in\mathcal{K}$}
        \State $\Delta T_k\leftarrow \textsc{TraceDelta}(H_0,\ldots,H_{k-1},k)$
        \State $\Delta m_k\leftarrow \Delta T_k/n$
    \EndFor
    \State \Return $\{\Delta m_k\}_{k\in\mathcal{K}}$

\end{algorithmic}
\end{algorithm}

\begin{algorithm}[t]
\caption{\textsc{TraceDelta}: Trace-Power Change from Small Matrices}
\label{alg:trace-delta}
\begin{algorithmic}[1]
    \Require Small matrices $H_0, \ldots, H_{k-1}$, moment order $k$
    \Ensure Exact trace-power change $\Delta T_k$
    \State $\Delta T_k \leftarrow 0$
    \For{$\ell=1,2,\ldots,k$}
    \For{each $(a_1,\ldots,a_\ell)\in\mathcal{A}_{k-\ell,\ell}$}
        \State $M\leftarrow I$
        \For{$q=1,\ldots,\ell$}
            \State $M\leftarrow M H_{a_q}$
        \EndFor
        \State $\Delta T_k\leftarrow \Delta T_k+\frac{k}{\ell}\operatorname{Tr}(M)$
    \EndFor
\EndFor
\State \Return $\Delta T_k$
\end{algorithmic}
\end{algorithm}